\documentclass[10pt,journal]{IEEEtran}
\usepackage{cite}
\usepackage{amsmath,amssymb}
\usepackage{subfigure}
\usepackage[hidelinks]{hyperref}
\usepackage{graphicx}
\usepackage{xcolor,soul}
\usepackage{algorithm,algpseudocode}
\usepackage{amsthm}
\usepackage{accents}
\usepackage{bbold}
\usepackage{mathtools}
\usepackage{verbatim}
\usepackage{subfigure}
\usepackage{soul}
\usepackage{stackengine}
\usepackage{tikz}
\usepackage{MnSymbol}
\soulregister\cite7
\soulregister\ref7
\newcommand{\NewT}{\Omega}

\newcommand\munderbar[1]{%
  \underaccent{\bar}{#1}}
\renewcommand{\munderbar}{\underline}
\renewcommand{\bar}{\overline}

\newcommand{\colvec}[2][.8]{%
  \scalebox{#1}{%
    \renewcommand{\arraystretch}{.8}%
    $\begin{smallmatrix}#2\end{smallmatrix}$%
  }
}

\newtheorem{thm}{Theorem}

\newtheorem{prop}{Proposition}
\newtheorem{cor}{Corollary}
\theoremstyle{definition}
\newtheorem{defn}{Definition}

\newtheorem{ass}{Assumption}

\renewcommand{\baselinestretch}{1.0}

\newcommand\copyrighttext{
  \footnotesize This work has been submitted to the IEEE for possible publication. Copyright may be transferred without notice, after which this version may no longer be accessible.
}
\newcommand\copyrightnotice{
\begin{tikzpicture}[remember picture,overlay]
\node[anchor=south,yshift=10pt] at (current page.south) {\fbox{\parbox{\dimexpr\textwidth-\fboxsep-\fboxrule\relax}{\copyrighttext}}};
\end{tikzpicture}
}

\makeatletter
\long\def\@makecaption#1#2{\ifx\@captype\@IEEEtablestring%
\footnotesize\begin{center}{\normalfont\footnotesize #1}\\
{\normalfont\footnotesize\scshape #2}\end{center}%
\@IEEEtablecaptionsepspace
\else
\@IEEEfigurecaptionsepspace
\setbox\@tempboxa\hbox{\normalfont\footnotesize {#1.}~~ #2}%
\ifdim \wd\@tempboxa >\hsize%
\setbox\@tempboxa\hbox{\normalfont\footnotesize {#1.}~~ }%
\parbox[t]{\hsize}{\normalfont\footnotesize \noindent\unhbox\@tempboxa#2}%
\else
\hbox to\hsize{\normalfont\footnotesize\hfil\box\@tempboxa\hfil}\fi\fi}
\makeatother
\begin{document}
\title{Robust Stability of Neural Network-controlled Nonlinear Systems with Parametric Variability}

\author{Soumyabrata~Talukder,~\IEEEmembership{Student Member,~IEEE}
      Ratnesh~Kumar,~\IEEEmembership{Fellow,~IEEE}
\thanks{The work was supported in part by the National Science Foundation under the grants, CSSI-2004766 and PFI-2141084.}
\thanks{S. Talukder and R. Kumar are with the Department of Electrical and Computer Engineering, Iowa State University, Ames, IA 50010, USA (e-mail: talukder@iastate.edu, rkumar@iastate.edu).}
}  

\maketitle
\copyrightnotice
\begin{abstract}
Stability certification and identifying a safe and stabilizing initial set are two important concerns in ensuring operational safety, stability, and robustness of dynamical systems. With the advent of machine-learning tools, these issues need to be addressed for the systems with machine-learned components in the feedback loop. To develop a general theory for stability and stabilizability of a neural network (NN)-controlled nonlinear system subject to bounded parametric variation, a Lyapunov-based stability certificate is proposed and is further used to devise a maximal Lipschitz bound for the NN controller, and also a corresponding maximal region-of-attraction (RoA) inside a given safe operating domain. To compute such a robustly stabilizing NN controller that also maximizes the system's long-run utility, a stability-guaranteed training (SGT) algorithm is proposed. 
The effectiveness of the proposed framework is validated through an illustrative example.
\end{abstract}


\IEEEpeerreviewmaketitle
\section{Introduction}\label{intro}
\IEEEPARstart{A}{pplication} of NNs to control dynamical systems has gained attention following the recent architectural innovations in NN and the advancements in training algorithms. The NN controllers are trained either in a supervised way, often referred to as \emph{imitation learning} \cite{venkatraman2015improving,schaal2003computational}, or in a semi-supervised way in the form of \emph{reinforcement learning} (RL) \cite{sutton2018reinforcement}. ``Model-free'' RL methods allow data-driven learning of an optimal policy by interacting with the physical system and receiving a reward for each one-step action, without requiring explicit knowledge of the model, e.g., Q-learning \cite{mnih2015human}, and multiple versions of policy-gradient methods \cite{schulman2015trust,lillicrap2015continuous,mnih2016asynchronous,choromanski2018structured}. In contrast, ``model-based'' NN methods are feasible when a model of the physical system, to be used to train an NN controller, is either known or can be identified by interacting with the system \cite{clavera2018model,kurutach2018modelensemble}. The application of RL controllers in real-world critical infrastructures has commenced \cite{diao2020training}. 

Using NNs as controllers offers design flexibility owing to its ability to approximate a large class of Lipschitz functions \cite{hornik1989multilayer}. Yet their demonstrations are mostly restricted in simulated environments \cite{sallab2017deep,8787888,8600371}. One key reason is the lack of closed-loop stability assurance of systems under NN controllers trained using the above algorithms. Their stability analysis is challenging due to the inherent complexity of NN-based control policies \cite{heuillet2021explainability}. Also, while there exist algorithms involving a convex-relaxed search for finding a local optimal RoA of nonlinear systems \cite{ahmadi2019dsos, 8973914} and for quantifying the corresponding stability margin \cite{9198917}, such methods become computationally intractable when the control policy is based on an NN. These limitations led us to explore alternate ways to formally guarantee the stability of NN-controlled systems and compute their RoAs.

\subsection{Related Works}
In \cite{vamvoudakis2010online,8070470}, stability-assured RL algorithms are proposed, where the RL controllers are restricted to be linear and are learned through a gradient-based weight update. The input to such a controller is a set of manually crafted nonlinear basis of the system states; the selection of a set of effective basis for a given system is still an unsolved problem \cite{padhi2006single}. In \cite{7239628, 8319447}, the authors design a similar control scheme for nonlinear multi-agent systems. Also, for the aforementioned cases and others \cite{vamvoudakis2010online,8070470,padhi2006single,7239628,8319447,modares2014integral,7083712,mu2017novel}, the notion of stability is one of uniform ultimate boundedness of the state and/or output signals, whereas a method to ensure the safety of the entire state trajectory (so it remains contained within a given safe domain) has not been reported. Further, the above methods do not generalize for multilayered NN controllers with nonlinear activations due to the additional challenge of underlying nonconvexity in controller training.  

A few recent works exist in the literature \cite{8618996,9388885,9424176} which aim to address the problem of guaranteeing the stability of \emph{multilayered} NN-controlled nonlinear systems. However, the majority of these works study a linearized system, with the effect of nonlinearity and/or parametric uncertainty modeled as \emph{integral quadratic constraints} \cite{587335}. Among these, the method suggested in \cite{8618996} guarantees finite $\mathcal{L}_2$ gain with respect to an external disturbance and also computes a corresponding ``Lipschitz-like" upper bound for the NN controller. However, the designed controller fails to guarantee stability even in absence of any disturbance. 
In \cite{9388885}, the nonlinearity of an already trained NN controller is locally sector-bounded to attain asymptotic stability of a discrete-time system, and also to estimate an RoA in the form of a sub-level set of a Lyapunov function. While the method can {\em verify} the stability under a given controller, it cannot be used to \emph{synthesize} a stabilizing NN controller. In a later work \cite{9424176}, the authors propose an imitation learning-oriented SGT algorithm for NN controller synthesis, providing a convex stability certificate for a discrete-time system. However, its application is restricted to systems free from actuator nonlinearity and/or uncertainty since their presence introduces nonconvexity. Moreover, the suggested NN training algorithm solves a semidefinite program (SDP) at each NN parameter update step, making the training computationally expensive. 

Among other methods, an iterative counterexample-guided search for a Lyapunov function is introduced in \cite{chen2020learning, dailyapunov} to provide stability under ReLu-based NN controllers. The algorithm in \cite{chen2020learning} is guaranteed to converge in finite iterations, but the application domain is limited to piecewise linear discrete-time systems and cannot handle parametric variation. \cite{aydinoglu2020stability} shows that the  ReLu activation function can be represented as the solution of a linear complementarity problem, thereby casting the stability certification of a linear-complementarity system with a ReLu-based NN controller as a linear matrix inequality (LMI). \cite{9146733,9478933} introduced an ``actor-critic'' RL algorithm, where the critic NN is structurally constrained to be positive definite as desired of a Lyapunov function. In \cite{9638007}, an augmented random search-based ``soft safe'' RL algorithm is proposed that employs a corresponding penalty term to the policy NN's objective. None of these methods \cite{9638007,9146733,9478933} can yield a formal stability guarantee.

{\color{black}\subsection{Contributions}\label{sec:contribute}
For the class of state-feedback NN-controlled, locally continuously differentiable continuous-time (CT) nonlinear systems, subject to parametric variations within a known bound, we make the following key contributions:
\begin{itemize}
\item A Lyapunov-based sufficient condition is introduced to certify a system's local asymptotic stability, robust to arbitrary parametric variations, under a controller satisfying a certain Lipschitz bound.
\item An algorithm is introduced using the above result to compute a maximal Lipschitz bound such that any controller satisfying the bound locally is robustly stabilizing, and also a corresponding ``robust safe initialization set" (RSIS) that is a maximal RoA contained within a user-given safe operating domain (so that any initialization of the controlled-system within the RSIS guarantees that the state trajectory never leaves the safe domain and eventually converges at the system's equilibrium). 
\item An actor-critic RL algorithm is proposed to synthesize a multilayered NN controller satisfying the above Lipschitz bound and that also maximizes the system's expected utility with respect to random initializations and parametric variations.
\end{itemize}

Our stability condition is not limited to any special class of NN activation functions, unlike the studies in \cite{aydinoglu2020stability,dailyapunov,chen2020learning} that limit the activation to be ReLu. Further, unlike \cite{9146733,9478933,9638007}, our analysis is able to offer a formal closed-loop stability guarantee without requiring any a priori knowledge of a Lyapunov function, which is the restriction in \cite{osinenko2020reinforcement}. Further, in contrast to \cite{vamvoudakis2010online,8070470,padhi2006single,7239628, 8319447,modares2014integral,7083712,mu2017novel,8618996}, our method guarantees that the system's trajectory never leaves a given safe domain. Also, contrary to \cite{9388885} that only provides a stability verification result, our work also introduces a method for controller synthesis. Moreover, in contrast to \cite{9424176}, our stability condition allows nonlinearity and parametric variation in the actuator, and our proposed SGT of NN controllers does not suffer from solving a computationally expensive SDP at each update of NN parameters.}

\subsection{Organization and Notations}
In what follows, Section \ref{problem} briefs the problem statement and also provides an overview of the solution approach. Section \ref{sector_plant_control} presents the mathematical preliminaries, followed by our main stability theorem, which is then used to develop an algorithm to identify a class of robustly stabilizing NN-based controllers that attain a maximal common RSIS. Section \ref{comp} provides our RL algorithm to search for the stabilizing controller locally within the identified class, which maximizes a long-run expected utility. Section \ref{exmpl_combined} validates the proposed method through an illustrative example, and Section \ref{conclusion} concludes the paper.

\textit{Notations:} 
$\mathbb{R}$ (resp., $\mathbb{R}_{\geq0}$, $\mathbb{R}_{>0}$) denotes the real (resp., non-negative real, positive real) scalar field, $\mathbb{R}^n$ denotes the $n$-dimensional real vector field, and $\mathbb{R}^{m\times n}$ denotes the space of all real matrices with $m$ rows and $n$ columns. Operators $\leq,<,\geq,>$ on matrices or vectors indicate elementwise operation. For $x\in\mathbb{R}^n$, $x^i$ denotes its $i^{th}$ element, and $\left\|x\right\|_p$ denotes its $p$-norm for any real $p\geq1$. If $x$ is an $n$-length sequence of reals or $x\in\mathbb{R}^n$, $diag(x)$ denotes the $n\times n$ diagonal matrix, where the $i^{th}$ diagonal element is the $i^{th}$ element of $x$. For $M\in\mathbb{R}^{m\times n}$, its $(i,j)^{th}$ element is denoted by $M^{i,j}$
% , $M^{i}$ denotes its $i^{th}$ row,
and $M^T\in\mathbb{R}^{n\times m}$ denotes its transpose. 
For $M\in\mathbb{R}^{m\times n}$, $\left|M\right|\in\mathbb{R}^{m\times n}$ denotes the matrix comprising the elementwise absolute values, and 
if $M$ is square and symmetric (i.e., $m=n$ and $M=M^T$), $M\succcurlyeq\mathbf{0}$ (resp., $M\preccurlyeq\mathbf{0}$) denotes its positive (resp., negative) semidefiniteness. The Kronecker product of two matrices $M,N$ is denoted $M\odot N$. For a locally differentiable operator $f:\mathbb{R}^n\rightarrow\mathbb{R}^m$, $J_{f,x}\in\mathbb{R}^{m\times n}$ denotes its Jacobian matrix w.r.t. its operand $x\in\mathbb{R}^n$. $\mathbb{E}$ denotes the standard expectation operator. For a set $S$, $|S|$ denotes its cardinality.
Objects having symmetry are often abbreviated by introducing $*$, e.g., we abbreviate $x^TPx$ and
$\begin{bmatrix}
P_{11} & P_{21}^T\\
P_{21} & P_{22}
\end{bmatrix}$, 
respectively, as $x^TP[*]$ and
$\begin{bmatrix}
P_{11} & *\\
P_{21} & P_{22}
\end{bmatrix}$.

\section{Problem Statement and Solution Approach}\label{problem}

%{\color{black}\subsection{Problem Statement:}
We consider a controlled system of the following form:
\begin{equation}\label{clsys}
\begin{split}
    \dot{x}(t) &= f(x(t),u(t), \omega(t)),\\
    u(t) &= \pi(x(t)),
\end{split}
\end{equation} 
where $f:\mathbb{R}^{n}\times \mathbb{R}^{m}\times\mathbb{R}^d\rightarrow\mathbb{R}^n$ denotes the given nonlinear CT plant dynamics; $\pi:\mathbb{R}^n\rightarrow\mathbb{R}^m$ denotes a state-feedback control policy; $x(t)\in\mathbb{R}^n$, $u(t)\in\mathbb{R}^m$, and $\omega(t)\in\mathbb{R}^d$ respectively, denote the state, the control input, and dynamic parametric variable, at time $t\in\mathbb{R}_{\geq0}$. 
The $\omega$-values are assumed bounded within a set $\Theta\subset\mathbb{R}^d$ with $0\in\Theta$. %satisfying $-\infty< \theta_\ell \leq 0 \leq \theta_u < \infty$. 
Also a ``safe" operational domain $\mathcal{X}\subseteq \mathbb{R}^n$ containing the origin is specified; operating the system at any $x\notin\mathcal{X}$ is deemed unsafe, and hence must be avoided. For $\theta\in\Theta$, $x^*_\theta \in\mathbb{R}^n$ is an equilibrium of (\ref{clsys}) if $f(x^*_\theta, \pi(x^*_\theta), \theta) = 0$. As standardly assumed in literature \cite{9388885, vamvoudakis2010online,9424176,dailyapunov}, we assume that the equilibrium does not change with parameter variation, i.e., $x^*_\theta\equiv x^*$. Also, without loss of generality (WLOG), through a change of coordinates if needed, we take $x^*=0$ and $\pi(0)=0$. 

Let $\NewT$ denote the space of all 
$\mathbb{R}^d$-valued parametric evolutions $\omega:\mathbb{R}_{\geq0}\rightarrow\Theta$. For a $\omega\in\NewT$, $\omega^t:[0,t)\rightarrow\Theta$ denotes its ``$t$-prefix'', i.e., $\omega^t(\tau)=\omega(\tau)~\forall~\tau\in[0,t)$. 
The trajectory of (\ref{clsys}) under the parametric evolution $\omega\in\NewT$, when initialized at $x\in\mathbb{R}^n$, is denoted $\psi_{\pi}(\omega^{t},x)\in\mathbb{R}^n$ for any $t\in\mathbb{R}_{\geq0}$; its existence and uniqueness are assured under the following assumption:
\begin{ass}\label{ass:1} The plant dynamics $f(\cdot,\cdot,\cdot)$ is locally continuously differentiable. 
\end{ass}
\noindent Assumption \ref{ass:1} implies that $f(\cdot,\cdot,\cdot)$ is locally Lipschitz, which is sufficient for local \emph{existence and uniqueness} of $\psi_{\pi}(\omega,x(0))$ uniformly over $\omega\in\NewT$. This assumption also allows for a decomposition of the dynamics into a pair of additive linear and nonlinear parameter-dependent portions, with the latter possessing a ``sector bound" (as introduced later in Section~\ref{sector_plant_control}). The stability and safety-related notions used in this paper are introduced next:
\begin{defn}\label{stabdef}[Stable Equilibrium, Stabilizing Controller, Stabilizability, Stability, and Region-of-attraction.] For the system (\ref{clsys}) and the set of parametric evolutions $\NewT$, if exists a policy $\pi(\cdot)$ and a corresponding neighborhood $\mathcal{R}_{\pi,\NewT}$ of the origin such that uniformly over $\omega\in\NewT$:
\begin{equation}
    x \in \mathcal{R}_{\pi,\NewT} \Rightarrow 
        \lim_{t \to \infty} \psi_{\pi}(\omega^t,x)=0,
\end{equation}
then the origin is a \textbf{$\NewT$-stable equilibrium} under $\pi(\cdot)$; $\pi(\cdot)$ is a {\bf locally $\NewT$-stabilizing controller} (or simply $\NewT$-stabilizing controller); the system is \textbf{locally $\NewT$-stabilizable} (or simply $\NewT$-stabilizable); the controlled system is \textbf{locally $\NewT$-stable} (or simply $\NewT$-stable) under $\pi(\cdot$), and $\mathcal{R}_{\pi,\NewT}$ is a \textbf{$\NewT$-region-of-attraction} ($\NewT$-RoA) under $\pi(\cdot)$.
\end{defn}

An RoA under certain conditions serves as an RSIS defined next.
\begin{defn}\label{stabdef}[Robust Safe Initialization Set (RSIS).] For the given safe domain $\mathcal{X}$ and a $\NewT$-stabilizing controller $\pi(\cdot)$, if $\mathcal{S}_{\pi,\NewT}^{\mathcal{X}}\subseteq\mathcal{X}$ is a $\NewT$-RoA of system (\ref{clsys}) and satisfies the following: 
\begin{equation}
    x \in \mathcal{S}_{\pi,\NewT}^{\mathcal{X}} \Rightarrow 
        \psi_{\pi}(\omega^t,x)\in\mathcal{X}~\forall~t\in\mathbb{R}_{\geq0},
\end{equation}
then $\mathcal{S}_{\pi,\NewT}^{\mathcal{X}}$ is an RSIS. The space of all $\mathcal{S}_{\pi,\NewT}^{\mathcal{X}}$'s is denoted $\mathbb{S}_{\pi,\NewT}^{\mathcal{X}}$.
\end{defn}

We use the notion of Lipschitz bound to constrain a controller $\pi(\cdot)$, which is formalized below:
\begin{defn}\label{lipsdef}[Lipschitz function and bound.] A function $g:\mathcal{X}\rightarrow\mathcal{Y}$, where $\mathcal{X},~\mathcal{Y}$ are domains with $\left\|\cdot\right\|_{\infty}$ defined, is called \textbf{Lipschitz} w.r.t. $\left\|\cdot\right\|_{\infty}$ (or simply Lipschitz) if there exists $0\leq L<\infty$ satisfying:
\begin{equation}\label{lipdef}
    \left\|g(x_1)-g(x_2)\right\|_{\infty} < L\left\|x_1-x_2\right\|_{\infty}, ~\forall~x_1,x_2\in\mathcal{X},
\end{equation}
and $L$ is called a \textbf{Lipschitz bound}. 
\end{defn}
\noindent The set of state-feedback controls that evaluate to zero at the origin and are Lipschitz-bounded by $L\in\mathbb{R}_{\geq 0}$ is denoted $\Pi_L$.  

%\textit{\textbf{Objective}}: 
\subsection{Objective and Mathematical Formulation}
Given the system (\ref{clsys}) satisfying Assumption \ref{ass:1}, our first objective is to identify the class of state-feedback NN-based controllers so that any controller in that class is $\NewT$-stabilizing, and possesses a maximal common RSIS. Our next objective is to find an optimal NN-based controller in the identified class (which maximizes a long-run expected utility under random initializations and parametric variations). 
% In particular, we synthesize a controller $\pi^*(\cdot)$ along with its maximal RSIS $\mathcal{S}^{\mathcal{X}}_{\pi^*,\NewT}$, which is the solution of the following bi-objective optimization problem:
% \begin{align}\label{opt_con}
% \begin{split} \pi^* :=
% \underset{\colvec[.7]{\pi\in\Pi_{\NewT}}}{\text{argmax}} & \Bigg[\text{vol}(\mathcal{S}^{\mathcal{X}}_{\pi,\NewT}) +\mathfrak{w}.\Bigg(\mathop{\mathbb{E}}_{
%     \colvec[.7]{
%     \omega\sim\mathbb{P}(\NewT)\\
%     x\in\mathbb{P}(\mathcal{S}^{\mathcal{X}}_{\pi,\NewT})
%     }}
%     \Big[\lim_{\colvec[.7]{T \to \infty}}J_{\pi}(\omega^T,x)\Big]\Bigg)\Bigg]
% \end{split}
% \end{align}

% Due to nonconvexity of the space $\Pi_{\NewT}$ and presence of the two terms in the objective of (\ref{opt_con}) that are nonconvex with respect to $\pi(\cdot)$, in general it is computationally hard to solve for a global optimal solution $\pi^*(\cdot)$. Our solution strategy involves an alternative representation, in which 
%Here, restricting $\pi^*$ to be Lipschitz, we propose an approximate local optimal solution of (\ref{opt_con}). Let $\Pi_L$ denote the space of all state-feedback controllers that evaluate to zero at the origin and are Lipschitz-bounded by $L\in\mathbb{R}_{\geq 0}$. 
WLOG, a controller $\pi(\cdot)$ is written as a superposition of a linear gain ``nominal controller'' $\pi_{K}(x):=K.x$ for some $K\in\mathbb{R}^{m\times n}$ and an additive ``perturbation controller" $\pi_\rho:\mathbb{R}^n\rightarrow\mathbb{R}^m$ around the nominal one, to be implemented via an NN having parameter $\rho$, i.e., $\pi=\pi_{K}+\pi_\rho$. Then for the first objective, we compute an optimal linear state-feedback gain $K^*\in\mathbb{R}^{m\times n}$ for the nominal controller and a maximal Lipschitz bound $L^*\in\mathbb{R}_{\geq 0}$ for the perturbation controller such that the corresponding RSIS  $\mathcal{S}^*$ is maximal:
\begin{align}\label{stage_1}
\begin{split}  K^*, L^*, \mathcal{S}^* :=&
\underset{\colvec[1.0]{K\in\mathbb{R}^{m\times n},\\L\in\mathbb{R}_{\geq 0},~\mathcal{S}\subseteq\mathcal{X}}}{\text{argmax}}\Big[\text{vol}
(\mathcal{S}) +\mathfrak{w}.L\Big]\\
&\text{s.t. }
\mathcal{S}
\in\bigcap_{\pi_\rho\in\Pi_L} \mathbb{S}_{(\pi_K+\pi_\rho),\NewT}^{\mathcal{X}},
\end{split}
\end{align}
where for a compact set $\mathcal{S}\subseteq\mathbb{R}^n$, vol$(\mathcal{S}):=\int_{\mathcal{S}} 1dx$ denotes its volume, and
$\mathfrak{w}\geq0$ is a tunable ``trade-off'' parameter. Note the objective is to maximize vol$(\mathcal{S})$ to have a maximal RSIS (the fact that it is a common RSIS is ensured by the constraint $\mathcal{S}\in\bigcap_{\pi_\rho\in\Pi_{L^*}} \mathbb{S}_{(\pi_{K}+\pi_\rho),\NewT}^{\mathcal{X}}$) and also to maximize $L$ to have the largest possible search space for the candidate NN controllers. When the solution set $\mathcal{S}^*$ is non-empty, a state-feedback controller $\pi=\pi_{K^*}+\pi_\rho$ is $\NewT$-stabilizing for any $\pi_\rho\in\Pi_{L^*}$. To achieve the first objective, we develop a sufficient condition of $\NewT$-stabilizability of (\ref{clsys}) in Section~\ref{stability}, which extends the existing Lyapunov-based stability results.

For the next objective, the optimal NN controller $\pi_{\rho^*}\in\Pi_{L^*}$ is designed (so that the overall optimal controller is $\pi^*=\pi_{K^*}+\pi_{\rho^*}$) to maximize an expected utility as defined next. For $\omega\in\NewT$, initial state $x\in\mathcal{S}^*$, %$x\in\mathcal{S}^{\mathcal{X}}_{\pi,\NewT}$, 
a reward function $r:\mathbb{R}^n\times\mathbb{R}^m\rightarrow\mathbb{R}$, and time horizon $T\in\mathbb{R}_{\geq 0}$, let the $T$-horizon expected utility $J_{\pi}(\omega^T,x)\in\mathbb{R}$ be:
\begin{equation}\label{utility}
    J_{\pi}(\omega^T,x):=\int_{0}^{T} r(\psi_{\pi}(\omega^t,x),\pi(\psi_{\pi}(\omega^t,x))) dt.
\end{equation}
Then the optimal perturbation controller $\pi_{\rho^*}\in\Pi_{L^*}$ 
is computed by solving the following optimization problem:
\begin{equation}\label{stage_2}
 \pi_{\rho^*} :=
\underset{\colvec[1.0]{\pi_\rho\in\Pi_{L^*}}}{\text{argmax}} \Bigg[\mathop{\mathbb{E}}_{
    \colvec[.7]{
    \omega\sim\mathbb{P}(\NewT),\\
    x\in\mathbb{P}(\mathcal{S}^*)
    }}
    \Big[J_{\pi_{K^*}+\pi_\rho}(\omega^T,x)\Big]\Bigg],
\end{equation}
where the distributions $\mathbb{P}(\NewT),\mathbb{P}(\mathcal{S}^{\mathcal{X}}_{\pi,\NewT})$ in (\ref{stage_2}) are taken to be uniform in case those are unknown. 
%The first stage of the above formulation, i.e. (\ref{stage_1}), finds an optimal linear state feedback gain $K^*\in\mathbb{R}^{m\times n}$ for the nominal controller and a maximal Lipschitz bound $L^*\in\mathbb{R}_{\geq 0}$ for the perturbation controller such that  $\mathcal{S}^*\in\bigcap_{\pi_\rho\in\Pi_{L^*}} \mathbb{S}_{\pi_{K^*}+\pi_\rho,\NewT}^{\mathcal{X}}$ is the maximal RSIS, and when this is non-empty, a state-feedback controller $\pi=\pi_{K^*}+\pi_\rho$ is $\NewT$-stabilizing for any $\pi_\rho\in\Pi_{L^*}$. 
% It should be noted that the search space in (\ref{stage_1}) is convex, and its objective is simplified compared to (\ref{opt_con}) by replacing its nonconvex expected utility term with a linear budget in the form of a
%The Lipschitz bound $L$ in any case needs to be maximized to allow for a larger search space for a candidate NN-controller, whereas $(K,L,\mathcal{S})$ is optimized to also attain a maximal RSIS. The optimal perturbation controller $\pi^*_\rho$ is next designed in the second stage (\ref{stage_2}) to maximize the expected utility, using $(K^*,L^*,\mathcal{S}^*)$ resulting from the first stage.  
\begin{figure}[htbp]
\centering
%\vspace*{-.15in}
\includegraphics[width=1.7in]{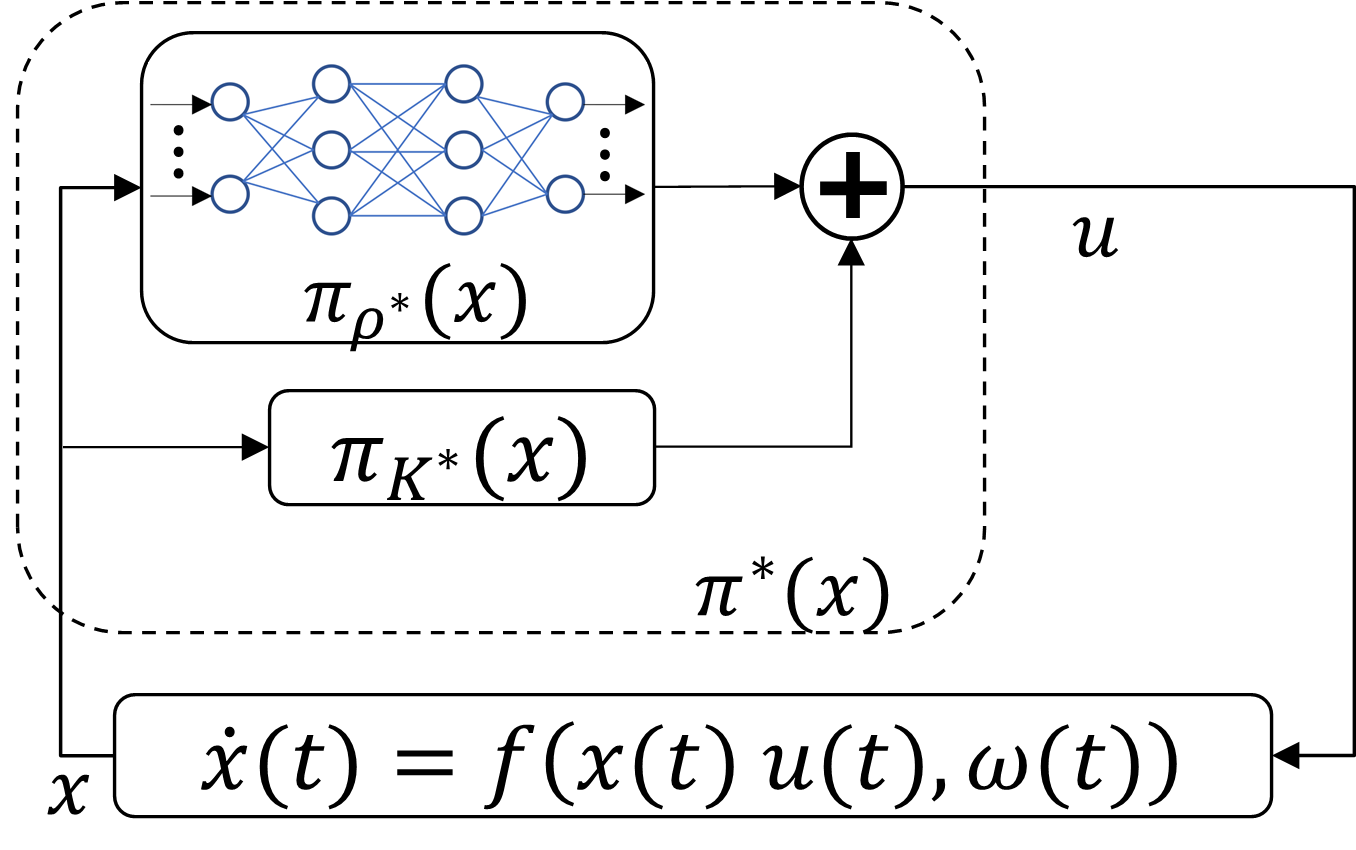}
% \vspace{-.1in}
\caption{Block diagram of control architecture}
\label{control_arch}
% \vspace*{-.15in}
\end{figure}
A schematic of the overall control architecture is shown in Fig.~\ref{control_arch} and a high-level flow-chart of the proposed overall method is shown in Fig. \ref{flow_chart_oa}.\\ 
\begin{figure}[htbp]
\centering
% \vspace*{-.1in}
\includegraphics[width=2.3in,trim={0cm 0cm 1cm 0cm}]{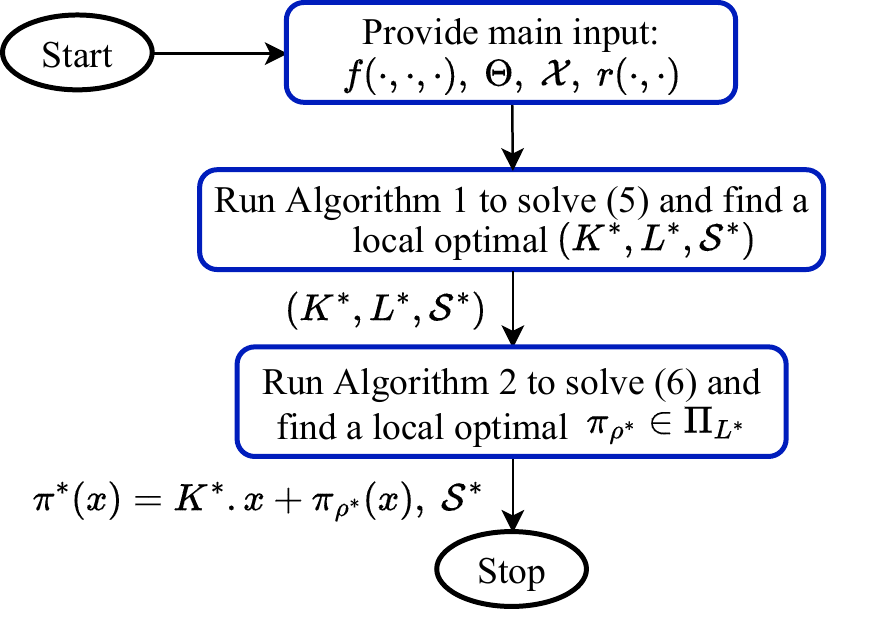}
% \vspace*{-.1in}
\caption{Flow-chart of the proposed solution approach}
\label{flow_chart_oa}
\end{figure}

Note (\ref{stage_1}) and (\ref{stage_2}) are both nonconvex. We propose Algorithm~\ref{algo1} in Section~\ref{optlibbound} to iteratively find a local optimal $(K^*,L^*,\mathcal{S}^*)$ solving (\ref{stage_1}). %relying on our sufficient condition of $\NewT$-stability from Section \ref{stability} that extends existing Lyapunov-based stability results. 
To find a local optimal control $\pi_{\rho^*}\in\Pi_{L^*}$ solving (\ref{stage_2}), Algorithm \ref{algo2} is proposed in Section \ref{comp}, which extends the traditional actor-critic RL \cite{mnih2016asynchronous} to attain an SGT of the NN controller by way of ensuring its Lipschitz boundedness. %The technical details and parameters of these algorithms are described later in their respective sections. 

\section{Optimal Nominal Control, Maximal Lipschitz Bound for NN Controller, and Maximal RSIS}\label{sector_plant_control}
To enable $\NewT$-stability analysis of the system (\ref{clsys}), we introduce in Section \ref{decom} an equivalent representation of (\ref{clsys}) in the form of a linear system, perturbed by a ``nonlinear and parameter variation (NPV)'' component, appearing as an additive term. A quadratic constraint (QC) that a Lipschitz-bounded controller $\pi_\rho\in\Pi_L$ necessarily satisfies is presented in Section \ref{nnsb0}.  In Section \ref{nnsb}, we introduce the notion of ``local $(\munderbar{\mathcal{L}},\bar{\mathcal{L}})$-sector'' to characterize a bound for the NPV. A method to compute the sector-defining parameters $(\munderbar{\mathcal{L}},\bar{\mathcal{L}})$ is also presented, and a necessary condition for the NPV to satisfy such a bound in the form of a QC is developed. In Section \ref{stability}, given a Lipschitz bound for $\pi_\rho$, a sector bound for the system NPV, and a safe operating domain $\mathcal{X}\subset\mathbb{R}^n$, a sufficient condition of $\NewT$-stability of system (\ref{clsys}) is introduced by extending Lyapunov's theory employing the above QCs. This is subsequently used in Section \ref{optlibbound} to develop an algorithm to iteratively search for a solution of (\ref{stage_1}).

\subsection{An Equivalent Representation of the Nonlinear System}\label{decom}
Following Assumption \ref{ass:1}, let $(A_\theta, B_{\theta})$ represent the linearized dynamics of the plant in (\ref{clsys}) at the origin for a certain parameter value $\theta\in\Theta$, where, respectively, the state and the input matrices $A_{\theta}\in\mathbb{R}^{n\times n}, B_{\theta}\in\mathbb{R}^{n\times m}$ under zero control are defined as: $A_{\theta}:=J_{f,x}\big|_{\colvec[0.7]{x=0\\u=0}}$ and $B_{\theta}:=J_{f,u}\big|_{\colvec[0.7]{x=0\\u=0}}$. Then the nonlinear dynamics under a state-feedback control $u=K.x+u_\rho$ for a $K\in\mathbb{R}^{m\times n}$ and a $u_\rho\in\mathbb{R}^m$ can be written as:
\begin{equation}
    f(x,K.x+u_\rho,\omega)=A_{0,K}.x+B_{0}.u_\rho+\eta_{K}(x,u_\rho,\omega),
\end{equation}
where the pair $(A_{0,K},B_{0})$ denotes the linearized dynamics of (\ref{clsys}) at the origin with parameter value $\theta=0$ under the feedback control 
$u=K.x+u_\rho$. In other words, $A_{0,K}:=(J_{f,x}+J_{f,u}.J_{u,x})\big|_{\colvec[0.7]{x=0,\theta=0\\u_\rho=0}}\equiv A_0+B_0.K$. 
Further the %$J_{f,u_\rho}\big|_{\colvec[0.7]{x=0,\theta=0\\u_\rho=0}}=J_{f,u}\big|_{\colvec[0.7]{x=0,\theta=0\\u=0}}=B_0$, and 
additive perturbation term is simply the difference: 
\[\eta_{K}(x,u_\rho,\omega):=f(x,K.x+u_\rho,\omega)-A_{0,K}.x-B_{0}.u_\rho,\] 
that is $\theta$-dependent. % and in the special cases of locally infinitely differentiable $f(\cdot,\cdot,\cdot)$, includes all nonlinear terms of Taylor's expansion of $f(x,\pi_K(x)+u_\rho,\theta)$ around $f(0,0,0)=0$.
$\NewT$-stability of the system (\ref{clsys}) under a state-feedback controller $u(x)=K.x+\pi_\rho(x)$ is then equivalent to $\NewT$-stability of the following system: 
\begin{equation}\label{cl_linear}
\begin{split}
    &\dot{x}(t)=A_{0,K}.x(t)+\underbrace{B_{0}.u_\rho(t)+\eta_{K}(x(t),u_\rho(t),\omega(t))}_{\text{NPV}:~ \zeta_K(x(t),u_\rho(t),\omega(t))},\\
    &u_\rho(t)=\pi_\rho(x(t))
\end{split}
\end{equation}
where the effect of the parametric variation and the nonlinearities underlying $f(\cdot,\cdot,\cdot)$ and $u_\rho(\cdot)$ is viewed as a disturbance 
\[\zeta_{K}(x,u_\rho,\theta):=f(x,K.x+u_\rho(x),\theta)-A_{0,K}.x,\] 
additive to the linear system $\dot{x}=A_{0,K}.x$ that we refer to as the ``{\em nominal system}''. 

\subsection{Quadratic condition from Lipschitz-bounded Control}\label{nnsb0}
% Here we introduce the notion of $L$-Lipschitz ($L\in\mathbb{R}_{>0}$) NN controller. A necessary condition of the same in form of a QC is also provided.
For an NN-based perturbation controller $\pi_\rho\in\Pi_L$, we define the notion of ``$L$-bounded control-subspace'' based on its Lipschitz-boundedness property:
\begin{defn}
[$L$-bounded control-subspace.] For a Lipschitz bound $L\in\mathbb{R}_{\geq0}$ and a domain $\mathcal{X}\subset\mathbb{R}^n$, the {\bf $L$-bounded control-subspace} $\mathcal{U}_{L,\mathcal{X}}\subset\mathbb{R}^m$ of a controller $\pi_\rho\in\Pi_L$ is: 
\begin{equation}\label{smoothcs}
    \mathcal{U}_{L,\mathcal{X}}\!\!:=\!\!\big\{u_\rho\!\!\in\!\!\mathbb{R}^m~\!\!\big|\!~\exists~\! x\!\in\!\mathcal{X}\!:\!\pi_\rho(x)\!=\!u_\rho,\left\|u_\rho\right\|_{\infty}\!\leq\! L\left\|x\right\|_{\infty}\!\!\big\}.
\end{equation} 
\end{defn}

Next, we  provide a necessary condition for a controller $\pi_\rho(\cdot)\in\Pi_L$ to be Lipschitz-bounded by $L$, in form of a QC, which is a variation of Lemma 4.2 of \cite{8618996}:
\begin{prop}\label{qcpi}
For a Lipschitz constant $L\in\mathbb{R}_{\geq0}$, let $\pi_{\rho}(\cdot)\in\Pi_{L}$ be a  controller (with $\pi_{\rho}(0)=0$). Then there exists $\chi:\mathbb{R}^n\rightarrow\mathbb{R}^{mn}$  satisfying $\chi(0)=0$, such that:
\begin{equation}\label{eq:lipschitz1}
    \pi_{\rho}(x) = \underbrace{[\mathbf{I}_m\odot\mathbf{1}_{1\times n}]}_{:=Q}.\chi(x),
\end{equation}
and the following QC globally holds for all $\gamma_{i.j}\geq0$ $\forall~i\in1,\ldots,m$, $j\in1,\ldots,n$:
\begin{equation}\label{qcmat}
    \begin{bmatrix}
        x\\
        \chi
    \end{bmatrix}^T
    \begin{bmatrix}
        L^2.diag(\{\Gamma_j\}) & \mathbf{0}_{n\times mn}\\
        * & diag(\{-\gamma_{i.j}\})
    \end{bmatrix}
    \begin{bmatrix}
    *
    \end{bmatrix}
    \geq0,
\end{equation}
where $\Gamma_j:=\sum_{i=1}^m\gamma_{i.j}$. 
\end{prop}
\begin{proof} The proof is provided in Appendix \ref{qcmatproof}.
\end{proof}
\subsection{Bound on Nonlineariy and Parametric Variation}\label{nnsb}

To characterize a bound of the NPV $\zeta_K(\cdot,\cdot,\cdot)$ in (\ref{cl_linear}), we introduce the notion of ``local $(\munderbar{\mathcal{L}},\bar{\mathcal{L}})$-sector'':
\begin{defn}[Local $(\munderbar{\mathcal{L}},\bar{\mathcal{L}})$-sector.] For a $K\in\mathbb{R}^{m\times n}$, a Lipschitz bound $L\in\mathbb{R_{\geq0}}$, and matrices $\munderbar{\mathcal{L}},\bar{\mathcal{L}}\in\mathbb{R}^{n\times (n+m)}$ satisfying $\munderbar{\mathcal{L}}\leq\bar{\mathcal{L}}$, the NPV $\zeta_{K}(x,u_\rho,\theta)$ of system (\ref{cl_linear}) under a controller $\pi_\rho\in\Pi_L$ is said to be {\bf locally $(\munderbar{\mathcal{L}},\bar{\mathcal{L}})$-sector} bounded over $\mathcal{X}\subset\mathbb{R}^{n}$, if the following:
\begin{equation}\label{secquad1}
\begin{split}
    &\munderbar{\mathcal{L}}^{i,j}\leq J_{\zeta_K,x}^{i,j}\big|_{\begin{smallmatrix*}[l]
    &x=\hat{x}\\
    &\theta=\hat{\theta}\\
    &u_\rho=\hat{u}
    \end{smallmatrix*}}\leq\bar{\mathcal{L}}^{i,j},~\begin{matrix*}[l]
    \forall~ i,j\in1,\ldots,n\\
    \end{matrix*}, \text{and}\\
    &\munderbar{\mathcal{L}}^{i,j+n}\leq J_{\zeta_K,u_\rho}^{i,j}\big|_{\begin{smallmatrix*}[l]
    &x=\hat{x}\\
    &\theta=\hat{\theta}\\
    &u_\rho=\hat{u}
    \end{smallmatrix*}}\leq\bar{\mathcal{L}}^{i,j+n},~\begin{matrix*}[l]
    \forall~ i\in1,\ldots,n\\
    \forall~j\in1,\ldots,m
    \end{matrix*}
\end{split}
\end{equation}
holds uniformly $\forall~\hat{x}\in\mathcal{X}$, $\hat{\theta}\in\Theta$, and $\hat{u}\in\mathcal{U}_{L,\mathcal{X}}$, where $\mathcal{U}_{L,\mathcal{X}}\subset\mathbb{R}^m$ denotes the $L$-bounded control-subspace corresponding to $\mathcal{X}$. 
\end{defn}

\noindent\textbf{\textit{Computation of $(\munderbar{\mathcal{L}},\bar{\mathcal{L}})$-sector}}: Recall $\zeta_K(x,u_\rho,\theta)=f(x,K.x+u_\rho,\theta)-A_{0,K}.x$, and so $J_{\zeta_K,x}=J_{f,x}+J_{f,u}.J_{u,x}-A_{0,K}=J_{f,x}+J_{f,u_\rho}.K-A_{0,K}$ and $J_{\zeta_K,u_\rho}=J_{f,u_\rho}$. Thus following Assumption \ref{ass:1}, under which
$J_{f,x}$ and $J_{f,u_\rho}$ are well-defined locally, so are also $J_{\zeta_K,x}$ and $J_{\zeta_K,u_\rho}$. Then
the $(i,j)^{th}$ element of the sector defining matrices, given a $K\in\mathbb{R}^{m\times n}$, a $L\in\mathbb{R}_{\geq0}$, and a $\mathcal{X}\subset\mathbb{R}^n$ can be computed as below. $\forall~i,j\in1,\ldots,n$:
\begin{equation}\label{bounds1}
    \begin{split}
        \munderbar{\mathcal{L}}^{i,j} & := \inf_{\begin{smallmatrix*}[l]
    \hat{x}\in\mathcal{X},\hat{u}\in\mathcal{U}_{L,\mathcal{X}}\\
    \theta\in\Theta
    \end{smallmatrix*},}\bigg(J^{i,j}_{\zeta_K(\cdot,\cdot,\cdot),x}\Big|_{\begin{smallmatrix*}[l]
    x=\hat{x}\\
    u_\rho=\hat{u}
    \end{smallmatrix*}}\bigg),\\
    \bar{\mathcal{L}}^{i,j} & := \sup_{\begin{smallmatrix*}[l]
    \hat{x}\in\mathcal{X},\hat{u}\in\mathcal{U}_{L,\mathcal{X}}\\
    \theta\in\Theta
    \end{smallmatrix*},}\bigg(J^{i,j}_{\zeta_K(\cdot,\cdot,\cdot),x}\Big|_{\begin{smallmatrix*}[l]
    x=\hat{x}\\
    u_\rho=\hat{u}
    \end{smallmatrix*}}\bigg),\\
    \end{split}
\end{equation}
and $\forall~i\in\{1,\ldots,n\},~j\in\{1,\ldots,m\}$:
\begin{equation}\label{bounds2}
    \begin{split}
        \munderbar{\mathcal{L}}^{i,j+n} & := \inf_{\begin{smallmatrix*}[l]
    \hat{x}\in\mathcal{X},\hat{u}\in\mathcal{U}_{L,\mathcal{X}}\\
    \theta\in\Theta
    \end{smallmatrix*},}\bigg(J^{i,j}_{\zeta_K(\cdot,\cdot,\cdot),u_\rho}\Big|_{\begin{smallmatrix*}[l]
    x=\hat{x}\\
    u_\rho=\hat{u}
    \end{smallmatrix*}}\bigg),\\
    \bar{\mathcal{L}}^{i,j+n} & := \sup_{\begin{smallmatrix*}[l]
    \hat{x}\in\mathcal{X},\hat{u}\in\mathcal{U}_{L,\mathcal{X}}\\
    \theta\in\Theta
    \end{smallmatrix*},}\bigg(J^{i,j}_{\zeta_K(\cdot,\cdot,\cdot),u_\rho}\Big|_{\begin{smallmatrix*}[l]
    x=\hat{x}\\
    u_\rho=\hat{u}
    \end{smallmatrix*}}\bigg).\\
    \end{split}
\end{equation}
Note for simplicity, the infima (resp., suprema) in (\ref{bounds1})-(\ref{bounds2}) can be relaxed by replacing those with the respective lower (resp., upper) bounds at the cost of slight conservativeness to the sector. The value of each such bound can be computed to a desired degree of accuracy via a binary search using a satisfiability-modulo-theory (SMT) solver (such as dReal \cite{10.1007/978-3-642-38574-2_14}), wherein the constraints regarding a postulated lower/upper bound, the boundedness of state domain, the $L$-boundedness of control subspace, and the parametric set $\Theta$ get represented as the conjunction of certain first-order formulas over the reals.

Next, a necessary condition for $\zeta_K(x,u_\rho,\theta)$ to be $(\munderbar{\mathcal{L}},\bar{\mathcal{L}})$-sector bounded locally over $\mathcal{X}\subset\mathbb{R}^n$ is proposed, in form of a $(K,L,\mathcal{X},\Theta)$-dependent QC:% relating $x$, $u_\rho$, and $\zeta_K$:

\begin{prop}\label{equi} For a $K\in\mathbb{R}^{m\times n}$ and a $L\in\mathbb{R}_{\geq0}$, consider the $\mathbb{R}^{n}$-valued NPV $\zeta_K(\cdot,\cdot,\cdot)$ of system (\ref{cl_linear}) that is locally $(\munderbar{\mathcal{L}},\bar{\mathcal{L}})$-sector bounded over $\mathcal{X}\subset\mathbb{R}^n$ under a controller $u_\rho(t)=\pi_\rho(x(t))$ with $\pi_\rho(\cdot)\in\Pi_L$. Then for each $\theta\in\Theta$, exists $\xi_\theta:\mathbb{R}^n\rightarrow\mathbb{R}^{n(n+m)}$ satisfying $\xi_\theta(0)=0$, such that:
\begin{equation}\label{eq:lipschitz}
    \zeta_K\big(x,u_\rho,\theta) = \underbrace{[\mathbf{I}_n\odot\mathbf{1}_{1\times (n+m)}]}_{:=R}.\xi_\theta(x), ~\forall~x\in\mathcal{X}.
\end{equation}
Further for $i\in1,\ldots,n$ and $j\in1,\ldots,n+m$, let $c_{ij}:=(\munderbar{\mathcal{L}}^{i,j}+\bar{\mathcal{L}}^{i,j})/2$, $\bar{c}_{ij}:=\max(|\munderbar{\mathcal{L}}^{i,j}|,|\bar{\mathcal{L}}^{i,j}|)$, and $k_{ij}:=i+(j-1)n$. Then uniformly for any $\theta\in\Theta$, $\pi_\rho(\cdot)\in\Pi_L$, and $\Lambda\in\mathbb{R}^{n.(n+m)}\geq0$, the following locally holds:
\begin{equation}\label{plantqc}
    \begin{bmatrix}
        x\\
        \chi\\
        \xi_\theta
    \end{bmatrix}^T
    \begin{bmatrix}
     M_{x\Lambda} & \mathbf{0}_{n\times m.n} & N_{x\Lambda}\\
     * &  M_{\chi\Lambda} & N_{\chi\Lambda}\\
        
        * & * & M_{\xi\Lambda}
    \end{bmatrix}
    \begin{bmatrix}
    *
    \end{bmatrix}
    \geq0,~\forall~x\in\mathcal{X}, 
\end{equation}
where recall $u_\rho=\pi_\rho(x)=Q.\chi(x)$ from (\ref{eq:lipschitz1}), and the matrices $M_{x\Lambda}$, $M_{\chi\Lambda}$, $M_{\xi,\Lambda}$, $N_{x\Lambda}$, and $N_{\chi\Lambda}$ are as defined below:
\begin{equation}\label{kdef}
    \begin{split}
        &\begin{split}
        M_{x\Lambda} := diag\Big(\big(\sum_{i=1}^n\Lambda^{k_{ij}}(\bar{c}_{ij}^2-c_{ij}^2)~\big|~j\in1,\ldots,n\big)\Big),
        \end{split}\\
        &\begin{split}
        M_{\chi\Lambda} := Q^Tdiag\Big(\big(\sum_{i=1}^n&\Lambda^{k_{ij}}(\bar{c}_{ij}^2-c_{ij}^2)~\big|~j\in n+1,\ldots,n+m\big)\Big)Q,\\
        \end{split}\\
        &M_{\xi\Lambda} := diag\big(-\Lambda\big),\\
        &\begin{split}
        &N_{x\Lambda} := \big[D_{x,1}~D_{x,2}~\ldots~D_{x,n}\big], \text{where:}\\
        &D_{x,i}:=\Big[diag\Big(\big(\Lambda^{k_{ij}}.{c}_{ij}~\big|
        ~j\in1,\ldots,n\big)\Big)~~\mathbf{0}_{n\times m}\Big],
        \end{split}\\
        &\begin{split}
        &N_{\chi\Lambda} := Q^T.\big[D_{u,1}~D_{u,2}~\ldots~D_{u,n}\big], \text{where:}\\
        &D_{\chi,i}:=\Big[\mathbf{0}_{m\times n}~~diag\Big(\big(\Lambda^{k_{ij}}.{c}_{ij}~\big|
        ~j\in n+1,\ldots,n+m\big)\Big)\Big].
        \end{split}\\
    \end{split}
\end{equation}
\end{prop} 
\begin{proof} The proof is provided in Appendix \ref{smoothnessproof}.
\end{proof}

\subsection{Lyapunov-based $\NewT$-Stability Certification}\label{stability}
We begin by recalling some existing Lyapunov-based stability-related results: 

\begin{defn}\label{lyaplemma}[Common Lyapunov function.] Consider the system (\ref{clsys}) under a given controller $\pi(\cdot)$. A continuously differentiable function $V:\mathcal{X}\rightarrow\mathbb{R}_{\geq0}$, where $\mathcal{X}\subset\mathbb{R}^n$ is a compact domain containing the origin, is a {\bf common Lyapunov function} (CLF) if uniformly for each $\omega\in\NewT$: 
\begin{equation}\label{lapth}
    \begin{split}
      V(x)& > 0,~\dot{V}(x)< 0,~\forall~ x\in\mathcal{X}\setminus\{0\};\\
                V(0) & = \dot{V}(0)  = 0.
    \end{split}
\end{equation}
\end{defn}

It is known that if a CLF exists for the system (\ref{clsys}), then the system is $\NewT$-stable, $\pi(\cdot)$ is $\NewT$-stabilizing, and the origin is a $\NewT$-stable equilibrium \cite{VU2005405,liberzon2003switching}. However, in general, finding a $\pi(\cdot)$ and its corresponding CLF is challenging. 

Taking $\pi(\cdot)$ to be of form $\pi(x)=\pi_K(x)+\pi_\rho(x)$ for a $K\in\mathbb{R}^{m\times n}$ and a $\pi_\rho\in\Pi_L$, along with the QC characterizations of the bound of $\pi_\rho(\cdot)$ and the local $(\munderbar{\mathcal{L}},\bar{\mathcal{L}})$-sector bound of the NPV in  (\ref{cl_linear}) (see Sections \ref{decom}-\ref{nnsb}), enables an efficient search for a CLF as demonstrated next: We state our key theorem, that for a given $(K,L)\in\mathbb{R}^{m\times n}\times\mathbb{R}_{\geq0}$, enables the verification of whether a state-feedback controller $\pi=\pi_K+\pi_\rho$ is $\NewT$-stabilizing for the system (\ref{clsys}) uniformly for each $\pi_\rho\in\Pi_L$, by way of a convex search for a quadratic CLF. 
\begin{thm}\label{main_th}  
Given a $L\in\mathbb{R}_{\geq0}$ and a neighborhood of the origin $\mathcal{X}\subset\mathbb{R}^n$, consider the system in (\ref{clsys}) under a controller $\pi(x)=\pi_K(x)+\pi_\rho(x)$ satisfying Assumption \ref{ass:1}, where $K\in\mathbb{R}^{m\times n}$ and $\pi_\rho\in\Pi_L$, so that its equivalent representation of (\ref{cl_linear}) and a corresponding local $(\munderbar{\mathcal{L}},\bar{\mathcal{L}})$-sector bound for its NPV exist. Then the system is $\NewT$-stable at the origin, uniformly for each $\pi_\rho\in\Pi_L$, if exist $K\in\mathbb{R}^{m\times n}, P\succ 0, \Lambda\geq0$, and $\gamma_{ij}\geq0$ for all $i\in1,\ldots,m$, $j\in1,\ldots,n$ satisfying:
\begin{equation}\label{lmi}
        \begin{split}
        & \begin{bmatrix}
            V_{L,\{\Gamma_j\},P,K} & * & *\\
            \mathbf{0}_{m.n\times n} & M_{\chi\Lambda}-diag(\{\gamma_{ij}\}) & *\\
            N_{x\Lambda}^T+R^T.P & N_{\chi\Lambda}^T & M_{\xi\Lambda}
        \end{bmatrix}
        \prec0,
        \end{split}
    \end{equation}
where recall  $\Gamma_j=\sum_{i=1}^m\gamma_{i.j}$ and $V_{L,\{\Gamma_j\},P,K}$ is defined as:
\begin{equation}
    V_{L,\{\Gamma_j\},P,K} = M_{x\Lambda}+L^2.diag(\{\Gamma_j\}) + P.A_{0,K}+A_{0,K}^T.P
\end{equation}
\end{thm}
\begin{proof} 
See Appendix \ref{main_th_proof}.
\end{proof}

\noindent 
Recall the matrices $M_{x\Lambda}$, $M_{\chi\Lambda}$, $N_{x\Lambda}$, and $N_{\chi\Lambda}$ are derived from the $(\munderbar{\mathcal{L}},\bar{\mathcal{L}})$-sector, which reveals their inherent $(K,L,\mathcal{X}$, $\Theta)$-dependence. This dependence, along with the presence of the bilinear terms in (\ref{lmi}), makes the latter non-convex when both $K$ and $P$ are search variables. On the other hand, if a $K$ is given, (\ref{lmi}) becomes an LMI that can be solved efficiently, and the existence of a feasible $(P\succ0,\Lambda\geq0,\{\lambda_{ij}\geq0\})$ certifies the $\NewT$-stability  of (\ref{clsys}) with the corresponding $V(x)=x^T.P.x$ serving as a CLF. Our Algorithm~\ref{algo1} in the next section enables a local search for a quadruple $(K\in\mathbb{R}^{m\times n}, P\succ 0, \Lambda\geq0,\{\lambda_{ij}\geq0\})$ satisfying (\ref{lmi}).

\begin{cor}\label{roaex}[Existence of RSIS.]
Consider the setting of Theorem \ref{main_th}. If the LMI (\ref{lmi}) is feasible for a $P\succ0$, then exists $\sigma>\mathbb{R}_{>0}$ such that the ellipsoid $\mathcal{E}_{P,\sigma}:=\{x\in\mathbb{R}^n~|~x^TPx\leq \sigma\}$ is contained in a given safe domain $\mathcal{X}=\{x\in\mathbb{R}^n~|~a_i^T.x\leq b_i, i\in1,\ldots,n_{\mathcal{X}}\}$ and serves as an inner-estimate of the maximal RSIS of system (\ref{clsys}), uniformly for each $\pi_\rho\in\Pi_L$. 
\end{cor}
\begin{proof}
See Appendix \ref{roaex_proof}
\end{proof}

\subsection{Optimal Nominal Control, Maximal Lipschitz Bound for NN Controller, and Inner-estimate of Maximal RSIS}\label{optlibbound}
We employ Theorem \ref{main_th} and Corollary \ref{roaex} to devise an iterative method of solving (\ref{stage_1}) in Algorithm \ref{algo1}, which finds a locally Pareto optimal pair $(K^*,L^*)$ and an inner-estimate of its corresponding maximal RSIS $\mathcal{S}^*$, where for computational purposes, the parametric set $\Theta$ as well as the safe operational domain $\mathcal{X}$ are taken to be polytopic, with $\mathcal{X}:=\{x\in\mathbb{R}^n~|~a_i^T.x\leq b_i,i\in1,\ldots,n_\mathcal{X}\}$. The strategy is to find a $(K^*,L^*)$, a corresponding $P^*\succ 0$, and the largest sublevel-set of $\mathcal{X}$ denoted 
$\mathcal{X}^*:=\{x~|~a_i^T.x\leq \delta^*.b_i,i\in1,\ldots,n_\mathcal{X};\delta^*\in(0,1]\}$ such that (\ref{lmi}) is feasible. Next, following Corollary \ref{roaex}, the largest hyper-ellipse $\mathcal{E}_{P^*,\sigma^*}$ contained in $\mathcal{X}^*$ is output as an inner-estimate of $\mathcal{S}^*$.

We begin with $L=\delta=0$, i.e., with a linear controller (since $L=0$) and the safety sublevel-set restricted to the origin (since $\delta=0$), over which the nonlinear dynamics is equivalent to the linear dynamics $(A_\theta,B_\theta)$ under the control of a nominal linear controller $\pi(x)=\pi_K(x)=K.x$. The initialization of Algorithm~\ref{algo1} requires computing a polytopic bound of $(A_\theta,B_\theta)$ for any $\theta\in\Theta$. 
 Let $\mathcal{I}$ denote the set of indices of $\theta$-dependent elements in $(A_\theta~B_\theta)$. %, and also let $\mathcal{P}$ denote its powerset. 
 Note $|\mathcal{I}|\leq n^2+m.n$. For each $\wp\subseteq\mathcal{I}$, let $(A_\wp~B_\wp)$ be obtained by replacing the $\theta$-dependent elements of $(A_\theta,B_\theta)$ corresponding to the indices in $\wp$ (resp. $\mathcal{I}\setminus \wp$) with their respective upper (resp. lower) bounds over $\Theta$. Then for any $\theta\in \Theta$, $(A_\theta~B_\theta)$ belongs to the polytope with $(A_\wp~B_\wp)$'s as the vertices, i.e.:
\begin{equation}\label{polytope}
\forall~\theta\in\Theta:\big[A_\theta~~B_\theta\big]=\sum_{\wp\subseteq\mathcal I}\gamma_{\wp} \big[A_{\wp}~~B_{\wp}\big],
\end{equation}
where $\gamma_{\wp}\in[0,1]$ such that $\sum_{\wp\subseteq\mathcal I}\gamma_{\wp}=1$.
To find the vertices $(A_{\wp}, B_{\wp})$'s, the bounds of its respective $\theta$-dependent elements can be computed via an SMT solver-based search (similar to that for the elements of $(\munderbar{\mathcal{L}},\bar{\mathcal{L}})$ in Section \ref{nnsb}).

\noindent\begin{minipage}{\linewidth}
\begin{algorithm}[H]{\color{black}
    \caption{Iterative local-optimal solution of (\ref{stage_1})}
    \label{algo1}
    \textbf{Input}: The dynamic model $f(\cdot,\cdot,\cdot)$ and its parametric set $\Theta$, the trade-off parameter $\mathfrak{w}\in\mathbb{R}_{\geq0}$, the maximum iterative steps $n_{steps}$, and the safe domain: $\mathcal{X}=\{x~|~a_i^T.x\leq b_i,i\in1,\ldots,n_\mathcal{X}\}$.
    
    \textbf{Initialize}: $k=1$, $\Delta=1/n_{steps}$, $\delta^0=L^0=0$, $P^0=Q^{-1}$, $K^0=Y.Q^{-1}$, where $Q\in\mathbb{R}^{n\times n}$ and $Y\in\mathbb{R}^{m\times n}$ are found as: 
    \begin{equation}\label{lmilin}
    \begin{split} 
    &Q,Y =~
    \underset{Q\succ0,Y\in\mathbb{R}^{m\times n}}{\mathbf{argmax}}\hspace{3 mm} \ln(\det(Q))
    \renewcommand{\thefootnote}{$\dagger$}
    \footnotemark
    \\ 
    \mathbf{s.t.}~
    &Q.{A_{\wp}}^T+A_{\wp}.Q+B_{\wp}.Y+Y^T.B_{\wp}^T\prec0\\
    &~~~~~~~~~~~~~~~~~~~~\forall~{\wp\subseteq\mathcal{I}},\\
    &\|Q.a_i\|_2\leq b_i~\forall~i\in1,\ldots,n_{\mathcal{X}},
    \end{split}
    \end{equation}
    where $(A_{\wp},B_{\wp})$'s are such that (\ref{polytope}) holds.
    \begin{algorithmic}[1]
        \While{$k~\leq~ n_{steps}$}
        \State $\delta^k = \delta^{k-1} + \Delta,~
            L^k = L^{k-1} + \mathfrak{w}.\Delta$
        \State $\mathcal{X}^k=\{x~|~a_i^T.x\leq \delta^k.b_i,i\in 1,\ldots,n_\mathcal{X}\}$
        \State Compute $M_{x\Lambda}$, $M_{\chi\Lambda}$, $N_{x\Lambda}$, $N_{\chi\Lambda}$ using (\ref{kdef}) corresponding to $(K^{k-1},L^k,\mathcal{X}^k,\Theta)$ and find $K^+$:
        \begin{equation}\label{klmi}
        \begin{split} 
        K^+ =~&
        \underset{K\in\mathbb{R}^{m\times n},\Lambda\geq0,\big\{\gamma_{ij}\geq0\big|\colvec[0.8]{i\in1,\ldots,m,\\
        j\in1,\ldots,n}\big\}}{\mathbf{argmin}}\hspace{3 mm} \|K-K^{k-1}\|_2\\ 
        \mathbf{s.t.:}~
        & \text{LMI in } (\ref{lmi})\text{ given }P=P^{k-1}
        \end{split}
        \end{equation}
        \If{(\ref{klmi}) is Feasible,}
        \State Update $M_{x\Lambda}$, $M_{\chi\Lambda}$, $N_{x\Lambda}$, $N_{\chi\Lambda}$ for $(K^{+},L^k,\mathcal{X}^k,\Theta)$ and find $P^+$:
        \begin{equation}\label{plmi}
        \begin{split} 
        P^+ =&
        \underset{P\succcurlyeq0,\Lambda\geq0,\big\{\gamma_{ij}\geq0\big|\colvec[0.8]{i\in1,\ldots,m,\\
        j\in1,\ldots,n}\big\}}{\mathbf{argmin}}\|P-P^{k-1}\|_2\\ 
        \mathbf{s.t.:}~
        & \text{LMI in } (\ref{lmi}) \text{ given }K=K^{+} \\
        \end{split}
        \end{equation}
        \EndIf
        \If{(\ref{klmi}) is Infeasible $\mathbf{or}$ (\ref{plmi}) is Infeasible,}
            \State Set $k=k-1$ and $\mathbf{break}$ 
        \Else
            \State $~K^k=K^+$, $P^k=P^+$, $k\leftarrow k+1$\;
        \EndIf
        \EndWhile
    \end{algorithmic}
    \textbf{Output}: {$K^*=K^{k}$, $L^*=L^{k}$, $P^*=P^{k}$, and:
        \begin{equation}\label{rstar}
        \begin{split} 
        \sigma^* =~&
        \underset{\sigma\in\mathbb{R}_{\geq0},x\in\mathcal{X}^{k}}{\mathbf{max}}\hspace{3 mm} \sigma\\ 
        \mathbf{s.t.:}~
        & x^T.P^*.x\leq \sigma.
        \end{split}
        \end{equation}}}
% \vspace*{-.1in}
\footnotetext{$^\dagger$In (\ref{lmilin}), $\ln(\cdot)$ denotes natural logarithm of its real scalar argument, and for a square matrix $M$, $\det(M)$ denotes its determinant.}      
\end{algorithm}
% \vspace*{-.2in}
\end{minipage}

The algorithm is initialized with a $(K,P)$ found by the convex search of (\ref{lmilin}) employing $(A_{\wp}~B_{\wp})$'s as parameters, which ensures that $\pi_K(x)$ is  $\NewT$-stabilizing for any linear system $(A_\theta,B_\theta)$, uniformly under any parametric variations \cite[pp. 100-102]{boyd1994linear}. Starting from $\delta=0$ as we increase $\delta\in[0,1)$ by a fixed amount $1/n_{steps}$ in each iteration (where $n_{steps}$ is a user-specified maximum number of iterative steps), the size of the sublevel-set of the safe domain increases. Due to the nonconvexity of (\ref{lmi}) when finding $(K,P)$ together, we split the search into two successive convex problems (\ref{klmi})-(\ref{plmi}) in each iteration. In (\ref{klmi}), holding $P$ fixed at its most recent value, we search for a $K$ in the neighborhood of its most recent iterate, subject to (\ref{lmi}), while keeping the matrices $M_{x\Lambda}$, $M_{\chi\Lambda}$, $N_{x\Lambda}$, and $N_{\chi\Lambda}$ unchanged, i.e., ignoring the effect on their value due to a change in $K$ over its past iterate. Next in (\ref{plmi}), those missing effects are restored when searching for a feasible $P$ in the neighborhood of its most recent iterate, while keeping $K$ fixed at its most recent value. 

In an iteration, if (\ref{klmi})-(\ref{plmi}) are both feasible, then the resultant $K,P$ satisfy (\ref{lmi}) for the current $L$, $\delta$-sublevel set of $\mathcal{X}$, a $\Lambda\geq0$, and $\gamma_{ij}\geq0$ for all $i\in1,\ldots,m$ and $j\in1,\ldots,n$.
On termination of the iterative loop, $K^*$, $L^*$, $P^*$, and $\sigma^*$ are reported as the output, where $\sigma^*$ is found by solving (\ref{rstar}). $\mathcal{E}_{P^*,\sigma^*}$ defines the largest sublevel set of the CLF $V(x)=x^T.P.^*x$ contained in $\mathcal{X}^*$. The objectives of (\ref{lmilin}) and (\ref{plmi}), and the conic constraints of (\ref{lmilin}) ensure that $P^*$ results in a locally maximal  $\mathcal{E}_{P^*,\sigma^*}$ \cite[p. 414]{boyd2004}.

\section{Optimal NN Control and Stability-guaranteed Training}\label{comp}
In this section, given the output of Algorithm \ref{algo1}, i.e., given a $(K^*,L^*)$ and a corresponding $\mathcal{E}_{P^*,\sigma^*}$ that is the largest hyper-elliptical inner-estimate  of the maximal RSIS $\mathcal{S}^*$, our goal is to solve (\ref{stage_2}) to find the NN-based ``perturbation controller"  $\pi_{\rho^*}\in\Pi_{L^*}$ such that the overall controller $\pi^*(x)=K^*.x+\pi_{\rho^*}(x)$ maximizes the expected long-run utility of the closed-loop system (\ref{clsys}) under parametric variations $\omega\sim\mathbb{P}(\NewT)$ and random initializations $x(0)\sim\mathbb{P}(\mathcal{E}_{P^*,\sigma^*})$.
Our gradient descent-based SGT to search for a locally optimal $\rho^*$, which extends the traditional ``actor-critic" RL \cite{sutton2018reinforcement, mnih2016asynchronous}, is presented in Algorithm~\ref{algo2}.  It should be noted that although Algorithm \ref{algo2} extends actor-critic RL, the approach proposed in this paper is general enough to be applied to any machine-learning-based deterministic controller design algorithm (e.g., imitation learning \cite{venkatraman2015improving, schaal2003computational}, deterministic policy gradient-based RL methods including the ``off-policy'' ones \cite{lillicrap2015continuous,8946888}, etc.).

Let $x(k)\in\mathbb{R}^n$ and $u_\rho(k)\in\mathbb{R}^m$ respectively, denote the state and the NN-based perturbation control values at the $k^{th}$ discrete sample instant under a uniform sampling period $\tau\in\mathbb{R}_{>0}$, available for training the NN $\pi_{\rho}$. Then the integral involved in defining the system's utility in (\ref{utility}) can be approximated by the corresponding discrete sum. Accordingly, the ``value'' of a state $x\in\mathcal{E}_{P^*,\sigma^*}$ employing an NN controller $\pi_{\rho}(\cdot)$, denoted $v_{\pi_{\rho}}(x)\in\mathbb{R}$, is \cite{sutton2018reinforcement}:
\begin{equation}\label{value}
\begin{split}
    v_{\pi_{\rho}}(x)&:= \mathop{\mathbb{E}}_{
    \begin{smallmatrix}
    \omega\sim\mathbb{P}(\NewT)\\
    \end{smallmatrix}}
    \bigg[\large\sum_{k=0}^{\infty} r(x(k),u(k))~\Big|\\
    & ~\pi=\pi_{K^*}+\pi_{\rho},\\
    &~x(k)=\psi_{\pi}(\omega^{k.\tau},x),~u(k)=\pi(x(k))\bigg].
    \end{split}
\end{equation}
Then the optimal NN controller $\pi_{\rho^*}(\cdot)$ is characterized by Bellman's optimality condition \cite{lewis2009reinforcement}:
\begin{equation}
\begin{split}
    v^*(x(k))=&r\big(x(k),\pi_{K^*}(x(k))+\pi_{\rho^*}(x(k))\big) + \\ &\mathop{\mathbb{E}}_{
    \begin{smallmatrix}
    \omega\sim\mathbb{P}(\NewT)\\
    \end{smallmatrix}} v^*(x(k+1)),
\end{split}
\end{equation}
where $v^*(x):=\max_{\rho} v_{\pi_{\rho}}(x)$. 

\noindent\begin{minipage}[htbp]{3.5in}
\begin{algorithm}[H]
    \caption{Actor-critic RL with stability guarantee}
    \label{algo2}
    \textbf{Input}: Actor and critic NNs parametrized by $\rho$ and $\phi$, sampling interval $\tau\in\mathbb{R}_{>0}$, training step sizes $\alpha_{\rho},\alpha_{\phi}\in\mathbb{R}_{>0}$, a diagonal matrix $\Sigma\in\mathbb{R}^{m\times m}$ s.t. $\Sigma\geq0$, decay rate of exploration $\nu_d\in(0,1)$ and its minimum value $\nu_{min}\in(0,1)$, no. of training trajectories $n_t$, integers $n_s$ and $n_a$ s.t. $n_s.\tau=T$ and $n_a<n_s$, trade-off parameter $\beta\in\mathbb{R}_{\geq0}$, and as introduced before $f(\cdot,\cdot,\cdot)$, $(K^*,L^*)$, $r(\cdot,\cdot)$, $\mathbb{P}(\mathcal{E}_{P^*,\sigma^*})$, and $\mathbb{P}(\NewT)$.

    \textbf{Initialize}: Exploration coefficient $\nu=1$, trajectory count $e=1$, initialize $\rho$, $\phi$ in their respective parameter spaces. 
    
    \begin{algorithmic}[1]
    \While{$e\leq n_t$}
      \State Set gradients $d\rho=0$, $d\phi =0$, sample index $k=0$; 
      \State 
        Randomly choose $x(0)\sim\mathcal{E}_{P^*,\sigma^*}$ and $\omega\sim\mathbb{P}(\NewT)$;
      \While{$k<n_s$}
          \State \noindent\begin{minipage}{0.86\linewidth}
          Given $x(k)$, $\omega(t)$ $\forall~ t\in[k.\tau,(k+1).\tau)$, apply random control $u_\rho(k)\sim\mathcal{N}(\pi_{\rho}(x(k)),$ $\Sigma)$ through a zero-order hold to observe $x(k+1)$, and compute reward $r(k):=r(x(k),\pi_{K^*}(x(k))+u_\rho(k))$;
          \end{minipage} 
          \If{$k\geq n_a$}
              \State \noindent\begin{minipage}{0.78\linewidth}
                Compute $n_a$-step advantage:\\ $a(k):=\sum_{l=0}^{N-1}r(k-l)+\hat{v}_\phi(x(k+1))-\hat{v}_{\phi}(x(k-n_a+1))$;
              \end{minipage} 
            \State \noindent\begin{minipage}{0.78\linewidth}
                $d\rho\leftarrow[(k-n_a)d\rho+\{(u(k)-\pi_{\rho}(x(k)))^T$
                
                $.\Sigma^{-1}\nabla_{\rho}
                                \pi_{\rho}(x(k))\}a(k)]/(k-n_a+1)$;
                \renewcommand{\thefootnote}{$\ddagger$}
                \footnotemark
            \end{minipage} 
             \State \noindent\begin{minipage}{0.78\linewidth}
                $d\phi\leftarrow[(k-n_a)d\phi+\nabla_{\phi}
                \{\hat{v}_\phi(x(k-n+1))-\hat{v}_\phi(x(k+1))\}.a(k)]/(k-n_a+1)$; \renewcommand{\thefootnote}{$\ddagger$}
                \footnotemark
             \end{minipage} 
             \State $k\leftarrow k+1$;
          \EndIf
      \EndWhile
      \State $\rho\leftarrow\rho+\alpha_{\rho}(d\rho-\beta.\nabla_{\rho}L_{\pi_{\rho}})$;
      \If{$L_{\pi_{\rho}}>L^*$}
          \State \noindent\begin{minipage}{0.85\linewidth}
            $\rho\leftarrow\rho\bigg(\frac{L^*}{L_{\pi_{\rho}}}\bigg)^{\frac{1}{n_l}}$;
          \Comment{$n_l=~$\# layers in $\pi_{\rho}(\cdot)$}
          \end{minipage} 
      \EndIf
      \State $\phi\leftarrow\phi-\alpha_{\phi}.d\phi$, $\nu\leftarrow\max(\nu_{min}, \nu.\nu_d)$,~$\Sigma\leftarrow \nu.\Sigma$;
      \State $e\leftarrow e+1$;
    \EndWhile
    \end{algorithmic}
    \textbf{Output}: Local optimal parameter $\rho^*=\rho$ for actor NN.
\end{algorithm}
\footnotetext{$^\ddagger$For a scalar differentiable function $f(x)$ of $x\in\mathbb{R}^n$, $\nabla_xf(x_0)\in\mathbb{R}^n$ denotes the gradient of $f$ w.r.t. $x$ at $x=x_0$}
\end{minipage}

As commonly practiced, in Algorithm~\ref{algo2} 
the value function is approximated by the ``critic" NN denoted $\hat{v}_\phi(\cdot)$, while the ``actor" NN $\pi_{\rho}(\cdot)$ serves as the controller. Both NNs are jointly trained over $n_t$ number of training trajectories, each comprising $n_s$ number of discrete time-steps. To enable effective exploration of the control space, at each training step, 
we choose $u_\rho(k)$ randomly from the Gaussian distribution $\mathcal{N}(\pi_{\rho}(x(k)),\Sigma)$ with mean $\pi_{\rho}(x(k))$ and covariance matrix $\Sigma\in\mathbb{R}^{m\times m}$. $\Sigma$ is initialized as a user-specified non-negative diagonal matrix, the elements of which are uniformly scaled down as the training proceeds. At the end of the training, the deterministic NN controller $\pi_{\rho^*}$ is deployed as the optimal perturbation controller. 

To improve training robustness, the $n_a$-step average of the computed gradients is used as the estimate of the true gradient in contrast to a single-step gradient estimate. To ensure $\NewT$-stability of the overall controller $\pi$, we constrain the search space of the NN controller $\pi_{\rho}$ within $\Pi_{L^*}$ by the following means: (i) we add to the policy gradient a regularizer (see Line 13 of Algorithm \ref{algo2}) proportional to the change in Lipschitz bound $L_{\pi_{\rho}}\in\mathbb{R}_{\geq0}$ of $\pi_{\rho}(\cdot)$, estimated using the computationally efficient method of \cite{gouk2021regularisation} (with $\beta\in\mathbb{R}_{\geq0}$ serving as a weight), and (ii) the elements of $\rho$ are uniformly scaled if the parameter update in a training step results in $L_{\pi_{\rho}}>L^*$ (see Line 14-15 of Algorithm \ref{algo2}). %Since $L_{\pi_{\rho}}$ needs to be estimated at each parameter update step, we choose the method of \cite{gouk2021regularisation} for this, which does not require solving any optimization, and hence, is computationally efficient. Alternatively, a more sophisticated method \cite{NEURIPS2019_95e1533e, doi:10.1137/19M1272780} may be chosen to obtain less conservative estimates of $L_{\pi_{\rho}}$ at the cost of added computation time.  

% Algorithm \ref{algo2} naturally extends to ``asynchronous'' parallelization to reduce the training time, using the method suggested in \cite{mnih2016asynchronous}.

\section{Illustrative Example}\label{exmpl_combined}
To validate the correctness and effectiveness of our proposed method, we consider the following illustrative nonlinear system of the form (\ref{clsys}) possessing continuously differentiable dynamics (here, the $i^{th}$ element of $x\in\mathbb{R}^n$ is denoted $x_i$):
\begin{equation}\label{exmpl}
    \begin{split}
    \dot{x} &= \begin{bmatrix}
    -(1+\omega_1)x_2\\
    x_1 + (1+\omega_2)(x_1^2-1)x_2
    \end{bmatrix}+u,\\
    u&=\pi(x),
    \end{split}
\end{equation}
where $\omega=[\omega_1~\omega_2]^T$ denotes the vector of time-varying parameters bounded within the range $\Theta\equiv [-0.05,0.05]\times[-0.1,0.1]$. If $\pi(0)=0$, regardless of $\omega$-value, the origin is an equilibrium of the above system. Let the reward function and the safe domain of the system be respectively given as: $r(x,u)=-(x^T.x+0.1u^T.u)$ and a polytope $\mathcal{X}\subset\mathbb{R}^n$ with vertices $(0.3, 0.6), (0.1962, 0.8077), (-0.3375, 0.1406)$, $(-0.3375, -0.8523)$, $(0.3, -0.2727)$ as shown in Fig. \ref{rsis_plot}.

Our objective is to find a $\NewT$-stabilizing $\pi^*(\cdot)$ and a corresponding maximal RSIS $\mathcal{S}_{\pi^*,\NewT}^{\mathcal{X}}\subset\mathcal{X}$ so that the expected long-run utility of (\ref{utility}) is maximized under random parametric variation in 
$\Theta\equiv [-0.05,0.05]\times[-0.1,0.1]$ and state initializations within $\mathcal{S}_{\pi^*,\NewT}^{\mathcal{X}}$. 
As proposed, $\pi^*(x)=K^*.x+\pi_{\rho^*}(x)$ with $\pi_{\rho^*}\in\Pi_{L^*}$, where 
$(K^*,L^*)$ and the associated hyper-elliptical inner-estimate of the maximal $\mathcal{S}_{\pi^*,\NewT}^{\mathcal{X}}$ are found employing Algorithm~\ref{algo1}, and $\pi_{\rho^*}\in\Pi_{L^*}$ is implemented using an NN, with its parameter $\rho^*$ trained using Algorithm~\ref{algo2}, maximizing the expected utility.

\subsection{Computation of $(K^*,L^*)$ and Inner-estimate of $\mathcal{S}_{\pi^*,\NewT}^{\mathcal{X}}$}Following Assumption \ref{ass:1}, by linearizing the dynamics of (\ref{exmpl}) at $x=u=0$, the matrices $A_\theta,B_\theta$ for a $\theta\in\Theta$ are derived as: 
\begin{equation}\label{nominal_linear}
    A_\theta = \begin{bmatrix}
    0 & -(1+\theta_1)\\
    1 & -(1+\theta_2)
    \end{bmatrix},~
    B_\theta =  \begin{bmatrix}
    1  & 0\\
    0 & 1 
    \end{bmatrix}.
\end{equation}
Since 2 of 8 elements of $(A_\theta,B_\theta)$ are $\theta$-dependent, $2^2=4$  $(A_{\wp},B_{\wp})$ vertices are computed such that (\ref{polytope}) holds. Using those as parameters, we first solve (\ref{lmilin}) and find a feasible pair:
\begin{equation*}
    K^0 = \begin{bmatrix}
    -2.8299 & 0.3352\\
   1.9226 & -0.9035
    \end{bmatrix},~
    P^0 = \begin{bmatrix}
    3.6841 & -0.5629\\
    -0.5629 & 1.7448
    \end{bmatrix},
\end{equation*}
that certifies the $\NewT$-stabilizability of (\ref{exmpl}) along with the existence of a neighborhood of the origin as its $\NewT$-RoA under the $\NewT$-stabilizing controller $\pi_{K^0}(x)=K^0.x$. 

Next we initialize Algorithm \ref{algo1} with $(K^0,P^0)$ and run it using $\mathfrak{w}=1.1$ and over $n_{steps}=20$ iterations  to search for $(K^*,L^*)$. At any iteration $k\leq n_{steps}$, the elements of matrices $\munderbar{\mathcal{L}}^k$ (resp., $\bar{\mathcal{L}}^k$) are conservatively computed to the accuracy of 0.001 via binary search employing the SMT solver dReal \cite{10.1007/978-3-642-38574-2_14}. The solutions of the convex problems (\ref{klmi})-(\ref{plmi}) certify the $\NewT$-stability of the nonlinear system (\ref{exmpl}) under the control of $\pi^k=\pi_{K^k}+\pi_{\rho}$ for any $\pi_{\rho}\in\Pi_{L^k}$ and any initialization within $\mathcal{E}_{P^k,\sigma^k}\subseteq\mathcal{X}^k\subseteq\mathcal{X}$ obtained by solving (\ref{rstar}). The iterative loop continues for $n_{steps}=20$ iterations, yielding $(K^*,L^*)$ and the corresponding $P^*$ as:
\begin{equation}\label{lstardstar}
   \hspace*{-.1in} K^* \!=\! \begin{bmatrix}
    -2.9714 \!\!&\!\! -0.1204\\
   1.5924 \!\!&\!\! -2.1744
    \end{bmatrix}\!,~L^*=1.1,~
    P^* \!=\! \begin{bmatrix}
    3.8426 \!\!\!&\!\!\! -0.2612\\
    -0.2612 \!\!&\!\! 1.5241
    \end{bmatrix}\!.
\end{equation}
Also, the level value $\sigma^*=0.3272$ for defining the hyper-ellipse $\mathcal{E}_{P^*,\sigma^*}$ is computed solving (\ref{rstar}). 

Recall $V(x) = x^T.P^*.[*]$ is a CLF over $\mathcal{E}_{P^*,\sigma^*}$ for the system (\ref{exmpl}) under any controller $\pi=\pi_{K^*}+\pi_{\rho}$ with $\pi_{\rho}\in\Pi_{L^*}$, which certifies the $\NewT$-stability of the system at the origin according to Theorem \ref{main_th}. Also, $\mathcal{E}_{P^*,\sigma^*}$ serves as a hyper-elliptical inner-estimate of $\mathcal{S}_{\pi,\NewT}^{\mathcal{X}}$ following Corollary \ref{roaex}. The computed $\mathcal{E}_{P^*,\sigma^*}\subset\mathcal{X}$ is shown in Fig. \ref{rsis_plot}. 
\begin{figure}[htbp]
\centering
\includegraphics[width=2.2in]{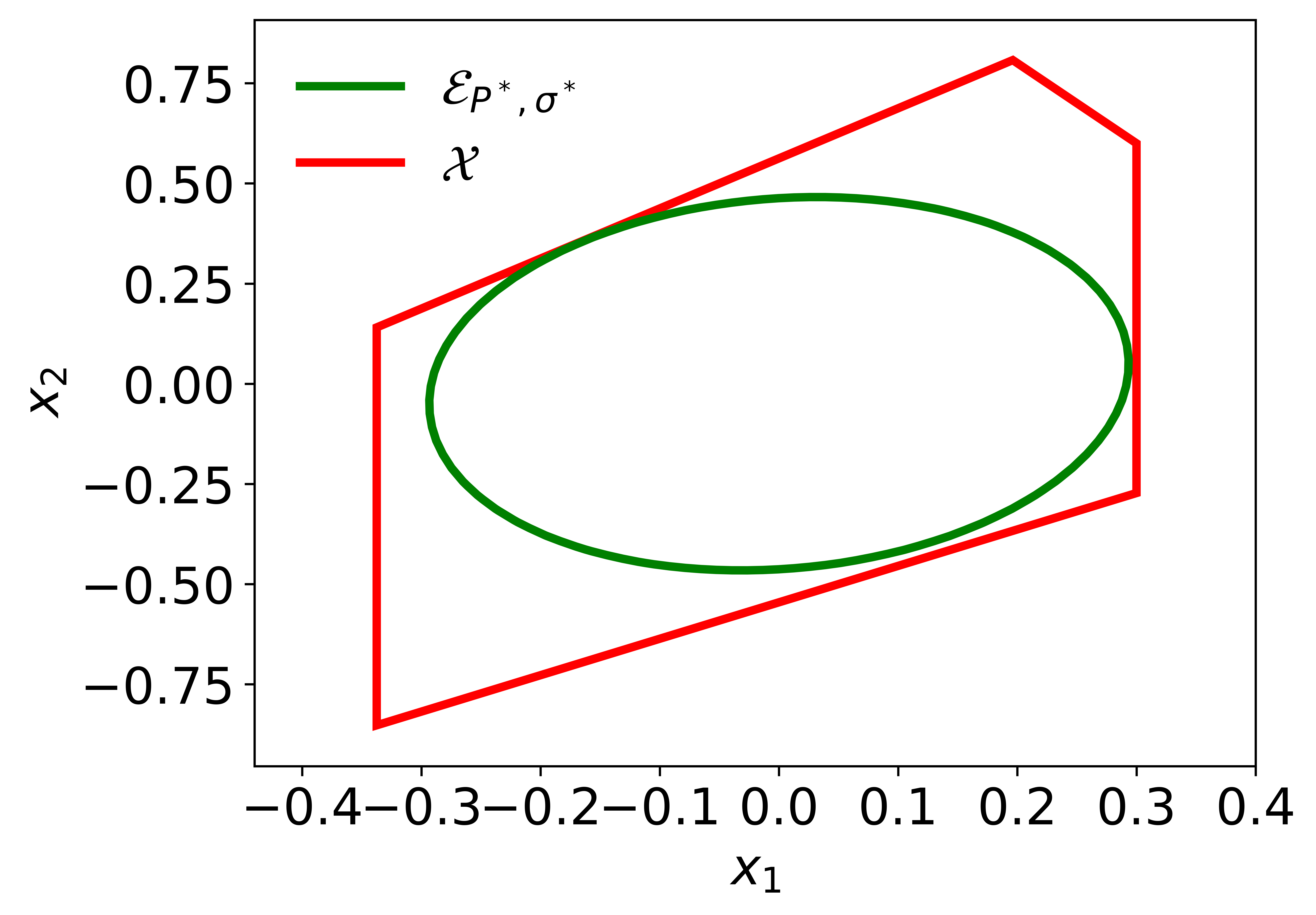}
\caption{Maximal elliptical RSIS inner-estimate $\mathcal{E}_{P^*,\sigma^*}\subset\mathcal{X}$}
\label{rsis_plot}
\end{figure}

To illustrate the principle underlying the algorithm, the evolution of the eigenvalues of $A_{0,K^k}$, i.e., the nominal system's state-matrix (see the representation of (\ref{cl_linear})), is shown in Fig. \ref{eigen_plot}. Clearly, as the permitted Lipschitz bound $L^k$ of the perturbation controller and the level $\delta^k$ of the safe domain increase over the successive iterations, $(K^k,P^k)$ get adjusted so that the eigenvalues are placed further away from the imaginary axis toward the left of the complex plane, thus securing higher ``margin of stability'' to allow larger NPV $\zeta_{K^k}$. 
\begin{figure}[htbp]
\centering
\includegraphics[width=2.0in]{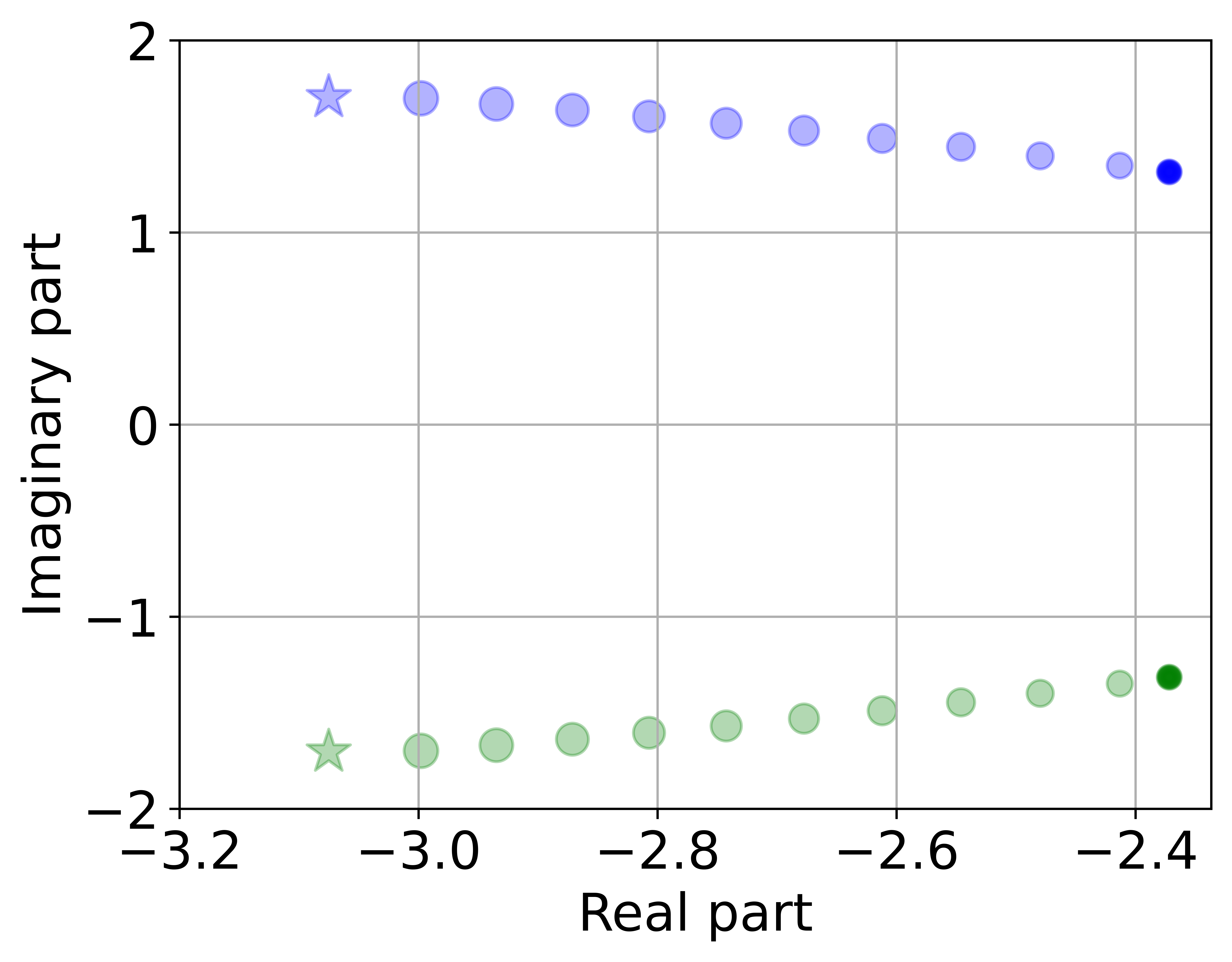}
\caption{Iterative progression of the eigenvalues of $A_{0,K^k}$. Bubbles of the two colors denote the respective two eigenvalues, whose sizes increase with the iterations; their final values are shown by the stars of the respective colors.}
\label{eigen_plot}
\end{figure}

This part of the algorithm is implemented in Python 3.7, and (\ref{lmilin})-(\ref{rstar}) are solved using CVXPY 1.2 with MOSEK 9.2.47 as the backend solver.  

\subsection{Computing NN Controller $\pi_{\rho^*}(\cdot)$}The NN controller is trained using Algorithm \ref{algo2}. Both the controller and value NNs, i.e., $\pi_{\rho}(\cdot):\mathbb{R}^2\rightarrow\mathbb{R}^2$ and $\hat{v}_{\phi}(\cdot):\mathbb{R}^2\rightarrow\mathbb{R}$ respectively, have two trainable layers. The hidden layer of each NN has 5 neurons, each with ``$tanh(\cdot)$'' activation, and the activation of the single neuron of the output layer is the identity function. The ``on-policy'' gradient decent for both the NNs are performed using Adam \cite{kingma2014adam} with step-size of $\alpha_{\rho}=\alpha_{\phi}=0.001$. The other parameters of the algorithm are set as follows: $\beta=10^{-15}$, $n_s=200$, $n_a=20$, $\tau=0.1$, $n_t=600$, $\nu_{d}=0.98$, $\nu_{min}=10^{-4}$, and each diagonal element in the diagonal covariance matrix $\Sigma$ is set to $(0.15)^2=0.0225$,. The trainable bias of the two layers for both $\pi_{\rho}(\cdot)$ and $\hat{v}_{\phi}(\cdot)$ are set to zero; this ensures $\pi_{\rho}(0)=0$. Upon termination of Algorithm \ref{algo2}, the trained weight matrices of the respective layers of the optimal NN controller $\pi_{\rho^*}(\cdot)$ are found to be:
\begin{equation*}
\begin{split}
    W_1=&\begin{bmatrix}
    -0.0503&-0.4911&0.4001&-0.2690&0.0077\\
    -0.3338&-0.2768&0.0496&0.3172&-0.1867
    \end{bmatrix}^T,\\
    W_2=&\begin{bmatrix}
    0.0119&0.0393&-0.3223&-0.2757&-0.1733\\
    0.1496&0.2292&0.1309&0.2942&0.2662
    \end{bmatrix}.
\end{split}
\end{equation*}
The Lipschitz bound of $\pi_{\rho^*}(\cdot)$ computed using the method proposed in \cite{gouk2021regularisation} is: $L_{\pi_{\rho^*}}=0.8218$, which is well below the value $L^*=1.1$ in (\ref{lstardstar}). Hence, by Theorem \ref{main_th}, the controller $\pi^*=\pi_{K^*}+\pi_{\rho^*}$ is $\NewT$-stabilizing for system (\ref{exmpl}) with $\mathcal{E}_{P^*,\sigma^*}$ as an inner-estimate of the maximal RSIS. 

This part of the algorithm is implemented in Python 3.7, and the architecture and backpropagation of the controller and value NNs are implemented using Tensorflow 2.3.

\subsection{Performance Evaluation of Trained Controller}

An instance of transient performance of the trained NN-based controller $\pi^*(x)$ is depicted in Fig. \ref{system_response}, where the parameters of the system are held fixed at $\omega_1\equiv\theta_1=-0.0253$, $\omega_2\equiv\theta_2=0.0532$, and the system is initialized at $x_1(0)=0.2752$, $x_2(0)=0.1866$. The response of the system under the above computed controller $u = \pi^*(x)$ is plotted. 
\begin{figure}[htbp]
\centering
\includegraphics[width=2.0in]{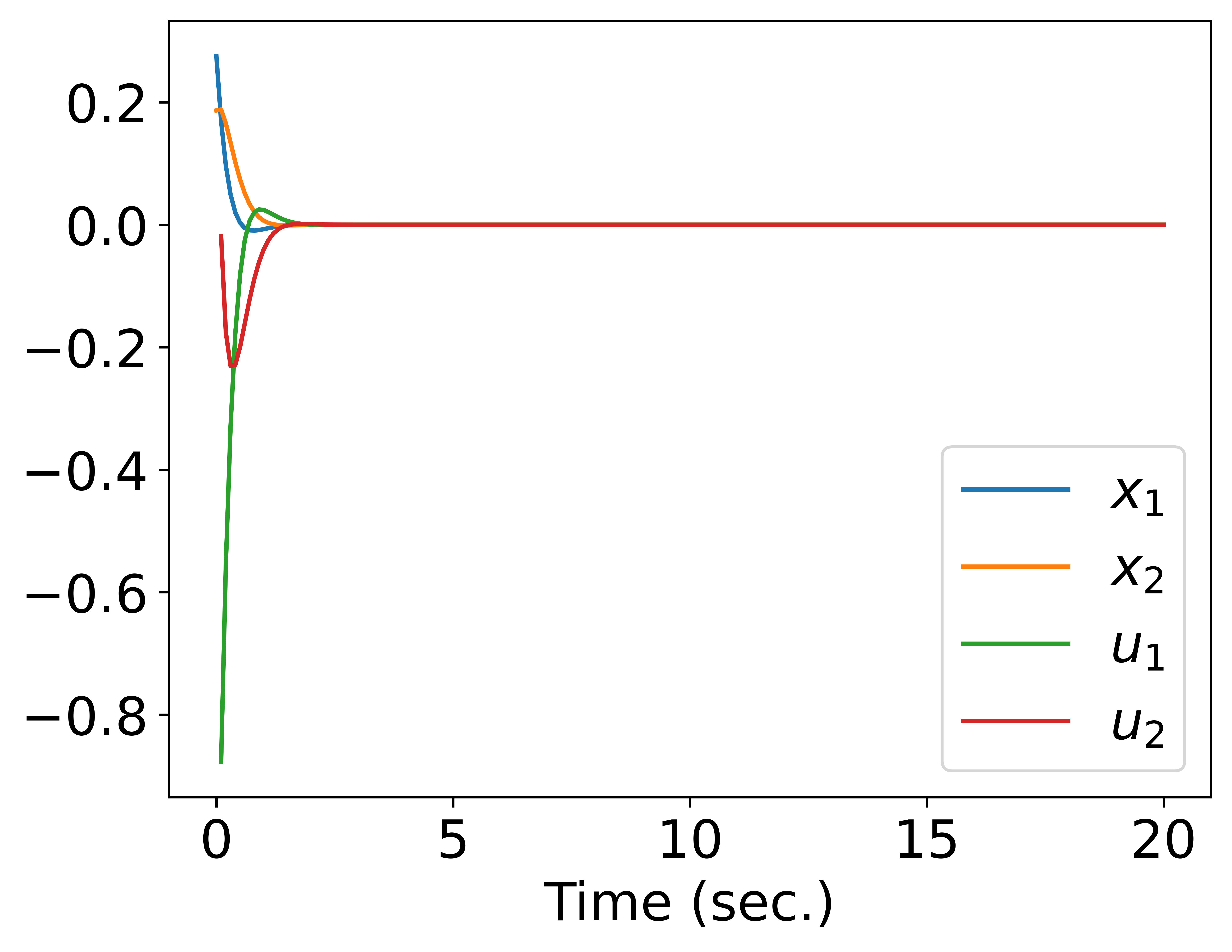}
\caption{System's transient response under $\pi^*(x)$, where the parameter value is: $\theta=[-0.0253,0.0532]^T$, and the initialization is at: $x(0)=[0.2752,0.1866]^T$}
\label{system_response}
\end{figure}

For a comparative validation of the performance of the proposed controller, we pick as benchmark the \emph{linear quadratic regulator} (LQR) designed for the linear nominal system. We compute the LQR gain for $(A_0,B_0)$ and the given reward function solving the algebraic Ricatti equation using MATLAB R2020b:
\begin{equation*}
    K_{LQR}=\begin{bmatrix}
    - 0.8350  & 0.1414\\
    0.1414 &   -0.5043
    \end{bmatrix}.
\end{equation*}
and set $K=K_{LQR}$, $u_\rho=0$ in the equivalent representation of (\ref{cl_linear}). While LQR can guarantee optimality and stability for the linear nominal dynamics whenever that is stabilizable and gets to be widely used even for the nonlinear systems, obtained against their local linearized models \cite{okyere2019lqr,chrif2014aircraft,10.1007/978-3-319-11933-5_48}; yet, in general, an estimate for the corresponding RoA is not available in the presence of plant nonlinearity and/or parametric variation. Additionally, LQR cannot guarantee the boundedness of the system's trajectory within $\mathcal{X}$ either.

 Next, we simulated 40 trajectories of the system's response, each with 200 discrete time steps at a sampling interval of $\tau=0.1$ sec., where $\{\omega(k.\tau)~|~k\in0,\ldots,n_s\}$ and $x(0)$ for each simulation were chosen uniformly randomly from their respective domains: $\Theta^{n_s+1}$ and $\mathcal{E}_{P^*,\sigma^*}$. For each selected $\{\omega(k.\tau)\}$ and $x(0)$, the system responses under LQR and also under the controller $\pi^*(\cdot)$ were simulated, and their utilities were computed using (\ref{utility}).   
The statistics of the utilities over these 40 simulations are shown in Fig. \ref{stat200}, where the median value of the utility slightly improved by 4.92\% under $\pi^*(\cdot)$ compared to that under LQR. Also, using our approach, the RSIS $\mathcal{E}_{P^*,\sigma^*}$ could also be computed as shown in Fig.~\ref{rsis_plot}, but that is not known for a typical LQR. In addition, note LQR computation is feasible only when $r(\cdot,\cdot)$ is quadratic, as chosen in this example, whereas our Algorithm \ref{algo2} does not have such restriction. 
\begin{figure}
    \centering
    \includegraphics[width=1.4in,trim={6.2cm 8.5cm 5.5cm 9cm}]{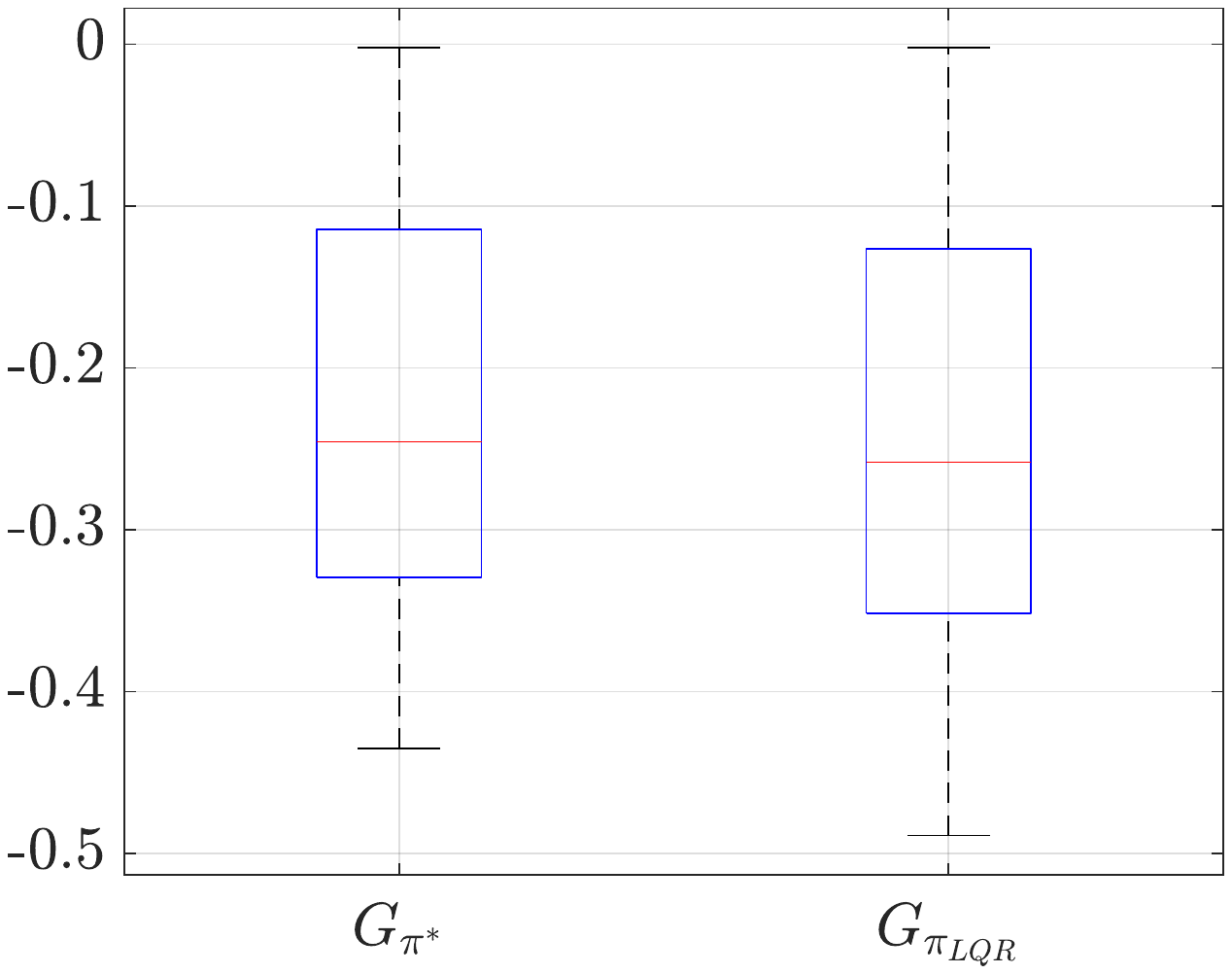}
    \caption{Box-whisker plots of $G_u$, for 40 simulations with $\omega(k.\tau)$, $x(0)$ chosen randomly, under $u=\pi_{LQR}(x):=K_{LQR}.x$ and under $u=\pi^*(x)$} respectively
    \label{stat200}
\end{figure}
\section{Conclusion}\label{conclusion}
The presented framework provides a way to design and certify NN (neural network) controllers for nonlinear systems subject to parameter variations for safety, stability, and robustness. Its a first framework for designing safe, stabilizing, and robust NN-based state-feedback controller for nonlinear continuous-time systems, where the dynamic model is known but is subject to unknown parametric variation over a given bounded set. A stability certificate is introduced extending the existing Lyapunov-based  results, and is further used to compute a maximal Lipschitz bound for a stabilizing NN-based controller, together with a corresponding maximal region-of-attraction contained in a user given safe operating domain, starting from where the asymptotic closed-loop stability of the system is guaranteed regardless of arbitrary parametric variation, and at the same time the state trajectory remains confined to the safe domain. A stability-guaranteed training algorithm is also presented to design such a safe  and robustly stabilizing NN controller that also maximizes the system's expected long-run utility, with respect to random initializations and parametric variations. The illustrative example validates the correctness of the proposed theory and the effectiveness of the proposed algorithms.  Future work can generalize the proposed framework for the case of partial observability and drifting equilibria under parametric variation. 
\bibliographystyle{IEEEtran}
\bibliography{IEEEabrv,References}

\section{Appendix}
\subsection{Proof of Proposition \ref{qcpi}}\label{qcmatproof}
\begin{proof}
It follows from (\ref{lipdef}) that a controller $\pi_\rho\in\Pi_L$ satisfies the following $\forall~x_1, x_2\in\mathbb{R}^n$:
\begin{equation}
\begin{split}
      \left\|\pi_\rho(x_1)-\pi_\rho(x_2)\right\|_{\infty} & \leq L\left\|x_1-x_2\right\|_{\infty}\\
      &\leq L\sum_{j=1}^n \left|x^j_1- x^j_2\right|.
\end{split}
\end{equation}
From the above, it further follows that there exists a set of functions: $\delta_{ij}:\mathbb{R}^n\times\mathbb{R}^n\rightarrow[-L,L]$  $\forall~i\in\{1,...,m\}$, $\forall~j\in\{1,...,n\}$ such that $\forall~x_1,x_2\in\mathbb{R}^n$:
\begin{equation}\label{qcmatproof1}
    \pi_\rho(x_1)-\pi_\rho(x_2)=
    \begin{bmatrix}
    \sum_{j=1}^n\delta_{1j}(x_1,x_2).(x^j_1-x^j_2)\\
    \vdots\\
    \sum_{j=1}^n\delta_{mj}(x_1,x_2).(x^j_1-x^j_2)
    \end{bmatrix}.
\end{equation}
Also, since $\pi_\rho(0) = 0$, we get the following by setting $x_1=x$ and $x_2=0$ in (\ref{qcmatproof1}):
\begin{equation}
    \pi_\rho(x)=
    \begin{bmatrix}
    \sum_{j=1}^n\delta_{1j}(x,0).x^j\\
    \vdots\\
    \sum_{j=1}^n\delta_{mj}(x,0).x^j
    \end{bmatrix}=[\mathbf{I}_m\odot\mathbf{1}_{1\times n}].\chi(x),                   
\end{equation}
where for $k:=i+(j-1)m\in\{1,\ldots,mn\}$, 
% for $i\in\{1,\ldots,m\}$ and $j\in\{1,\ldots, n\}$, 
the $k^{th}$
% $k^{th}=(i+(j-1)m)^{th}$
element of $\chi(x)$ is defined as $\chi^k\equiv \chi^{i+(j-1)m}:=\delta_{ij}(x,0).x^j$. Note this implies $\chi(0)=0$, also since $\delta_{ij}(x,0)^2\leq L^2$, we get:
\begin{equation}
    \begin{split}
      &(\chi^{i+(j-1)m})^2\leq L^2(x^j)^2\\
      \Rightarrow & \sum_{i,j}\gamma_{i.j}L^2(x^j)^2-\sum_{i,j}\gamma_{i.j}(\chi^{i+(j-1)m})^2\geq 0~\forall~\gamma_{i.j}\geq 0\\
      \Rightarrow &
    \begin{bmatrix}
        x\\
        \chi
    \end{bmatrix}^T
    \begin{bmatrix}
        L^2 diag(\Gamma_j) & \mathbf{0}_{n\times mn}\\
        * & diag(\{-\gamma_{i.j}\})
    \end{bmatrix}
    \begin{bmatrix}
    *
    \end{bmatrix}
    \geq0,
    \end{split}
    \nonumber
\end{equation}
\end{proof}
\subsection{Proof of Proposition \ref{equi}}\label{smoothnessproof}
\begin{proof}
To simplify notation, let us denote the space $\mathcal{X}\times\mathcal{U}_{L,\mathcal{X}}\subset\mathbb{R}^{n+m}$ by $\mathcal{Z}$, where $\mathcal{U}_{L,\mathcal{X}}$ is the $L$-bounded control subspace of a controller $\pi_\rho\in\Pi_L$ over $\mathcal{X}$. Accordingly, $(x,u_\rho)\in\mathcal{X}\times\mathcal{U}_{L,\mathcal{X}}$ is equivalently written as $z\in\mathcal{Z}$, where $z:=[x^T~~u_\rho^T]^T$. Also, the NPV $\zeta_K(x,u_\rho,\theta)$ is simply denoted $\zeta_K(z,\theta)$. Then $\forall~i\in\{1,\ldots,n\}$, $\forall~z_1,z_2\in\mathcal{Z}$, and for each $\theta\in\Theta$:
\begin{equation}\label{etaqc}
    \zeta_K^i(z_1,\theta) - \zeta_K^i(z_2,\theta)=\sum_{j=1}^{n+m}\Big\{\zeta_K^i(z_{2,j},\theta)-\zeta_K^i(z_{2,j-1},\theta)\Big\},
\end{equation}
where $z_{2,0}:=z_2$, and for $j>0$, the $k^{th}$ element of $z_{2,j}$ is:
\begin{equation}\label{splitup}
z^k_{2,j}:=
\begin{cases}
    z^k_1, & k\leq j\\
    z^k_2, & k > j
\end{cases}.
\end{equation}
Note that in the $j^{th}$ term of the summation in (\ref{etaqc}), the vectors $z_{2,j}$, $z_{2,j-1}\in\mathbb{R}^{m+n}$ are componentwise identical except for their $j^{th}$ component. This implies: $z_{2,j}-z_{2,j-1}=(z_1^j-z_2^j).1_j$, where $1_j\in\mathbb{R}^{m+n}$ is a binary vector with only the $j^{th}$ entry 1 and other entries zero. Since $\zeta_K(\cdot,\cdot,\cdot)$ is locally component-wise $(\munderbar{\mathcal{L}},\bar{\mathcal{L}})$-sector bounded over $\mathcal{X}$, we have from (\ref{secquad1}) that $\forall~i\in\{1,\ldots,n\}$ and $\forall~j\in\{1,\ldots,n+m\}$:
\begin{equation}
     \munderbar{\mathcal{L}}^{i,j}\leq J^{i,j}_{\zeta_K(z,\theta),z}\Big|_{\colvec[0.7]{z=\hat{z}\\\theta=\hat{\theta}}}\leq\bar{\mathcal{L}}^{i,j}, ~\forall~\hat{z}\in\mathcal{Z},~\forall~\hat{\theta}\in\Theta.
\end{equation}
It then follows that $\forall~i\in\{1,\ldots,n\}$, $\forall~j\in\{1,\ldots,n+m\}$, $\forall~\theta\in\Theta$, and for $z_{2,j},z_{2,j-1}$ as defined in (\ref{etaqc}):
\begin{equation}\label{eq:etabound}
     \munderbar{\mathcal{L}}^{i,j}(z_1^j-z_2^j)\leq \zeta_K^i(z_{2,j},\theta)-\zeta_K^i(z_{2,j-1},\theta)\leq\bar{\mathcal{L}}^{i,j}(z_1^j-z_2^j).
\end{equation}
Combining (\ref{etaqc}) and (\ref{eq:etabound}), we obtain $\forall~\theta\in\Theta$:
\begin{equation}\label{eq:etabound2}
     \sum_{j=1}^{n+m}\munderbar{\mathcal{L}}^{i,j}(z_1^j-z_2^j)\leq \zeta_K^i(z_1,\theta)-\zeta_K^i(z_2,\theta)\leq\sum_{j=1}^{n+m}\bar{\mathcal{L}}^{i,j}(z_1^j-z_2^j).
\end{equation}
This implies that for each $\theta\in\Theta$, there exists a set of functions: $\delta^{ij}_\theta:\mathcal{Z}\times\mathcal{Z}\rightarrow\big[\munderbar{\mathcal{L}}^{i,j},\bar{\mathcal{L}}^{i,j}\big]$ $\forall~i\in\{1,...,n\}$, $\forall~j\in\{1,...,n+m\}$ such that $\forall~z_1,z_2\in\mathcal{Z}$:
\begin{equation}\label{qcmatproof2}
    \zeta_K(z_1,\theta)-\zeta_K(z_2,\theta)=
    \begin{bmatrix}
    \sum_{j=1}^{n+m}\delta^{1j}_\theta(z_{1},z_{2}).(z^j_{1}-z^j_{2})\\
    \vdots\\
    \sum_{j=1}^{n+m}\delta^{nj}_\theta(z_{1},z_{2}).(z^j_{1}-z^j_{2})
    \end{bmatrix}.
\end{equation}
Also, since $\zeta_K(0,\theta) = 0$ $\forall~\theta\in\Theta$, we get the following for each $\theta\in\Theta$ by setting $z_1=z$ and $z_2=0$ in (\ref{qcmatproof2}):
\begin{equation}
    \zeta_K(z,\theta)=
    \begin{bmatrix}
    \sum_{j=1}^{n+m}\delta^{1j}_\theta(z,0).z^j\\
    \vdots\\
    \sum_{j=1}^{n+m}\delta^{nj}_\theta(z,0).z^j
    \end{bmatrix}=[\mathbf{I}_n\odot\mathbf{1}_{1\times (n+m)}].\xi_\theta(z),
\end{equation}
where for 
$k:=i+(j-1)n\in\{1,\ldots,n(m+n)\}$, the $k^{th}$
element of $\xi_\theta(z)$ is defined as, $\xi_\theta^k\equiv \xi_\theta^{i+(j-1)n}:=\delta^{ij}_\theta(z,0).z^j$. %Note the sequence of $k$, i.e. $\big(1,\ldots,n(m+n)\big)$, is formed by iterating over $i\in(1,\ldots,n)$ for each $j\in(1,\ldots,n+m)$.

From the definition of $\bar{c}_{i,j}$ and $c_{i,j}$, it follows that $\forall~i\in\{1,\ldots,n\}$ and $\forall~j\in\{1,\ldots,n+m\}$:
\begin{equation}\label{deltac}
    \begin{split}
        & |\bar{c}_{ij}|\geq|\delta^{ij}_\theta(z,0)-c_{ij}|\\
        \Leftrightarrow~& \bar{c}_{ij}^2(z^j)^2\geq\big(\delta^{ij}_\theta(z,0)z^j-c_{ij}z^j\big)^2\\
         \Leftrightarrow~& (\bar{c}_{ij}^2-{c}_{ij}^2).(z^j)^2+2c_{ij}.z^j.\xi_\theta^{i+(j-1)n}
             - \big(\xi_\theta^{i+(j-1)n}\big)^2\geq0
    \end{split}
\end{equation}
(\ref{deltac}) further implies that $\forall~\Lambda>0,k_{ij}:=i+(j-1)n$:
\begin{equation}\label{deltac1}
    \begin{split}
    \sum_{
    \begin{smallmatrix*}
    i\in\{1,\ldots,n\}\\
    j\in\{1,\ldots,n+m\}
    \end{smallmatrix*}}&\Lambda^{k_{ij}}\Big\{(\bar{c}_{ij}^2-{c}_{ij}^2).(z^j)^2~+\\
            &\mspace{15mu} 2c_{ij}.z^j.\xi_\theta^{k_{ij}} - (\xi_\theta^{k_{ij}})^2\Big\}\geq0.
        \end{split}
\end{equation}
Next, we get the following for each $\theta\in\Theta$, by writing (\ref{deltac1}) in matrix form, splitting variable $z$ into $x$ and $u_\rho$, and recognizing that $u_\rho=\pi_\rho(x)=Q.\chi(x)$ with $\pi_\rho(\cdot)\in\Pi_L$, for which $x\in\mathcal{X}\Rightarrow u_\rho\in\mathcal{U}_{L,\mathcal{X}}$:
\begin{equation}
    \begin{bmatrix}
        x\\
        \chi\\
        \xi_\theta
    \end{bmatrix}^T
    \begin{bmatrix}
     M_{x\Lambda} & \mathbf{0}_{n\times m.n} & N_{x\Lambda}\\
     * &  M_{\chi\Lambda} & N_{\chi\Lambda}\\
        
        * & * & M_{\xi\Lambda}
    \end{bmatrix}
    \begin{bmatrix}
    *
    \end{bmatrix}
    \geq0,~\forall~x\in\mathcal{X}, 
\end{equation}
where $M_{x\Lambda}$, $M_{q\Lambda}$, $M_{\xi\Lambda}$, $N_{x\Lambda}$, and $N_{q\Lambda}$ are as defined in (\ref{kdef}). 
\end{proof}
\noindent We note that the above proof is partially inspired from the proof of Lemma 4.2 of \cite{8618996}.

\subsection{Proof of Theorem \ref{main_th}}\label{main_th_proof}
\begin{proof} 
In the given setting, i.e., given $L\in\mathbb{R}_{\geq0}$, $\mathcal{X}\subset\mathbb{R}^n$, and the system (\ref{clsys}) under control of $\pi(x)=\pi_K(x)+\pi_\rho(x)$ satisfying Assumption \ref{ass:1}, assume that there exist $K\in\mathbb{R}^{m\times n}$, $P\succcurlyeq 0$, $\Lambda\geq0$, and $\gamma_{i,j}\geq0$ for all $i\in1,\ldots,m$, $j\in1,\ldots,n$ satisfying (\ref{lmi}), or equivalently, except at the origin the following holds:
\begin{equation}\label{lmi2}
    \begin{bmatrix}
        x\\
        \chi\\
        \xi_\theta
    \end{bmatrix}^T
    \begin{bmatrix}
           V_{L,\{\Gamma_{j},P,K\}} & * & *\\
            \mathbf{0}_{m.n\times n} & M_{\chi\Lambda} & *\\
            N_{x\Lambda}^T+R^T.P & N_{\chi\Lambda}^T & M_{\xi\Lambda}
        \end{bmatrix}
    \begin{bmatrix}
    *
    \end{bmatrix}
    <0.
\end{equation}
Also, owing to the local $(\munderbar{\mathcal{L}},\bar{\mathcal{L}})$-sector bound of the NPV of the equivalent system (\ref{cl_linear}), we can combine (\ref{cl_linear}) and Proposition \ref{equi} to get the following under a controller $u_\rho=\pi_\rho(x)$, uniformly $\forall~x\in\mathcal{X},\theta\in\Theta,\pi_\rho(\cdot)\in\Pi_L$:
\begin{equation}
    \dot{x}=f(x,u_\rho,\theta)\equiv f_\theta(x,u_\rho)=A_{0,K}.x+\underbrace{R.\xi_\theta(x)}_{=\zeta_K(x,u_\rho,\theta)}.
\end{equation}
Accordingly, by algebraic manipulation it follows that in the given setting, (\ref{lmi2}) is equivalent to the following, uniformly $\forall~x\in\mathcal{X},\theta\in\Theta,\pi_\rho(\cdot)\in\Pi_L$:
\begin{equation}\label{lmix}
    \begin{split}
        &\bigg\{\!x^TPf_\theta+f_\theta^TPx \!\bigg\} + \Bigg\{\!\!\begin{bmatrix}
        x\\
        \chi\\
        \xi_\theta
    \end{bmatrix}^T
    \!\begin{bmatrix}
        M_{x\Lambda} & \mathbf{0}_{n\times m.n} & N_{x\Lambda}\\
     * &  M_{\chi\Lambda} & N_{\chi\Lambda}\\
        * & * & M_{\xi\Lambda}
    \end{bmatrix}
    \!\begin{bmatrix}
    *
    \end{bmatrix}\!\!\Bigg\} \\
        & + \bigg\{\begin{bmatrix}
        x\\
        \chi
    \end{bmatrix}^T
    \begin{bmatrix}
        L^2 diag(\{\Lambda_j\}) & \mathbf{0}_{n\times mn}\\
        * & diag(\{-\lambda_{i.j}\})
    \end{bmatrix}
    \begin{bmatrix}
    *
    \end{bmatrix}
    \bigg\}<0.
    \end{split}
\end{equation}

From Proposition \ref{equi}, the local $(\munderbar{\mathcal{L}},\bar{\mathcal{L}})$-sector bound of the NPV of system (\ref{cl_linear}) also implies that uniformly $\forall~x\in\mathcal{X},\theta\in\Theta,\pi_\rho(\cdot)\in\Pi_L$, we have the second term of (\ref{lmix}) nonnegative. Moreover from Proposition \ref{qcpi}, $\pi_\rho(\cdot)\in\Pi_L$ implies that the third term is nonnegative.  Hence, in the given setting, uniformly $\forall~\theta\in\Theta,\pi_\rho(\cdot)\in\Pi_L$, (\ref{lmix}) is equivalent to:
\begin{equation}\label{fin}
    \begin{split}
        & x^T.P.f_\theta+f_\theta^T.P.x<0,~\forall~x\in\mathcal{X}\setminus\{0\}\\
        \Leftrightarrow~& \dot{V}(x)<0,~\forall~x\in\mathcal{X}\setminus\{0\},
    \end{split}
\end{equation}
where $V(x)=x^T.P.x$. It can be seen that $V(\cdot)$ is continuously differentiable and satisfies the conditions in (\ref{lapth}) over $\mathcal{X}$, regardless of how $\theta$ evolves over time. Hence, $V(x)$ is a CLF for (\ref{cl_linear}), and equivalently, also for system (\ref{clsys}) under controller $\pi(x)=\pi_K(x)+\pi_\rho(x)$, which implies that in the given setting, system (\ref{clsys}) is $\NewT$-stable, uniformly for $\pi_\rho(\cdot)\in\Pi_L$. 
\end{proof}
If the value of either of $K$ and $P$ is given, then note (\ref{lmix}) serves as a variant of ``S-procedure'' \cite[pp. 23-24]{boyd1994linear} used in various control applications to formulate conservative LMI relaxations for solving sets of indefinite QCs  \cite{8618996,9424176,9388885}.  
\subsection{Proof of Corollary \ref{roaex}}\label{roaex_proof}
\begin{proof}
Since the safe domain $\mathcal{X}\subset\mathbb{R}^n$ contains the origin, and also since $\mathcal{X}\subset\mathbb{R}^n$ is a neighborhood of the origin, there exists a $\sigma\in\mathbb{R}_{>0}$ s.t. the set $\mathcal{E}_{P,\sigma}$, which is a hyper-ellipse since $P\succ0$, is contained within both $\mathcal{X}$ and $\mathcal{X}$. 

In the given setting, i.e., given $L\in\mathbb{R}_{\geq0}$, $\mathcal{X}\subset\mathbb{R}^n$, and the system (\ref{clsys}) under control of $\pi(x)=\pi_K(x)+\pi_\rho(x)$ satisfying Assumption \ref{ass:1}, say $P$ satisfies (\ref{lmi}) for a certain $K\in\mathbb{R}^{m\times n}$. Then, following Theorem \ref{main_th}, since $V(x)=x^T.P.x$ is a CLF of system (\ref{clsys}) locally over $\mathcal{X}$, uniformly for any $\pi_\rho\in\Pi_L$, we have $\dot{V}(x)<0,~\forall~x\in\mathcal{X}\setminus\{0\}$. Also since $V(x)=\sigma$ uniformly over the boundary of $\mathcal{E}_{P,\sigma}\subset\mathcal{X}$, $\mathcal{E}_{P,\sigma}$ is an invariant set, i.e.,  
\begin{equation}
    x \in \mathcal{E}_{P,\sigma} \Rightarrow 
        \psi_{\pi}(\omega^t,x)\in\mathcal{E}_{P,\sigma},~\forall~t\in\mathbb{R}_{\geq0}.
\end{equation}
Hence, uniformly for each $x\in\mathcal{E}_{P,\sigma}\setminus\{0\}$, we have:
\begin{equation}
    \left\|P^{\frac{1}{2}}.\psi_{\pi}(\omega^{t'},x)\right\|_2 < \left\|P^{\frac{1}{2}}.\psi_{\pi}(\omega^t,x)\right\|_2,~\forall~t'>t.
\end{equation}
In other words, $\psi_{\pi}(\omega^{t},x)$ quadratically converges to the origin as $t\rightarrow\infty$ $\forall~\omega\in\NewT$ if the system is initialized within $\mathcal{E}_{P,\sigma}$. Hence, $\mathcal{E}_{P,\sigma}$ is a $\NewT$-RoA of the system (\ref{clsys}) under a controller $\pi(x)=\pi_K(x)+\pi_\rho(x)$, uniformly for any $\pi_\rho\in\Pi_L$. Using $\mathcal{E}_{P,\sigma}\subset\mathcal{X}$, it further follows that $\mathcal{E}_{P,\sigma}$ is an RSIS under a controller $\pi(x)=\pi_K(x)+\pi_\rho(x)$, uniformly for any $\pi_\rho\in\Pi_L$.
\end{proof}
\end{document}